\documentclass[lettersize,journal]{IEEEtran}
\usepackage{amsmath,amsfonts}
\usepackage{algorithmic}
\usepackage{algorithm}
\usepackage{array}
\usepackage[caption=false,font=footnotesize,labelfont=rm,textfont=rm]{subfig}
\usepackage{textcomp}
\usepackage{stfloats}
\usepackage{url}
\usepackage{verbatim}
\usepackage{graphicx}
\usepackage{cite}
\usepackage{amsthm}
\usepackage{hyperref}
\usepackage{multirow}
\usepackage{makecell}
\usepackage{array}
\usepackage{cleveref}
\usepackage{xcolor}
\usepackage{pifont}
\usepackage{colortbl}
\usepackage{booktabs}
\usepackage{amssymb}
\usepackage{tabularx}
\begin{document}

\title{Gaussian Shading++: Rethinking the Realistic Deployment Challenge of Performance-Lossless Image Watermark for
Diffusion Models}

\author{Zijin~Yang,
        Xin~Zhang,
        Kejiang~Chen,
        Kai~Zeng,
        Qiyi~Yao,
        Han~Fang,
        Weiming~Zhang,
        Nenghai~Yu
\thanks{Z. Yang, X. Zhang, K. Chen, Q. Yao, W. Zhang, N. Yu are with University of Science and Technology of China,
Hefei 230026, China and Anhui Province Key Laboratory of Digital Security (email: \{bsmhmmlf@mail., XinZhang1999@mail., chenkj@, qyyao@mail., zhangwm@, ynh@\}ustc.edu.cn).}
\thanks{K. Zeng is with the Department of Information Engineering and Mathematics, University of Siena, Siena, Italy (email: kai.zeng@unisi.it).}
\thanks{H. Fang is with the School of Computing, National University of Singapore, 117417, Singapore (email: fanghan@nus.edu.sg).}
\thanks{Zijin Yang and Xin Zhang contributed equally to this work.}
\thanks{Corresponding author: Kejiang Chen.}}

\maketitle

\begin{abstract}
Ethical concerns surrounding copyright protection and inappropriate content generation pose challenges for the practical implementation of diffusion models. One effective solution involves watermarking the generated images. Existing methods primarily focus on ensuring that watermark embedding does not degrade the model performance. However, they often overlook critical challenges in real-world deployment scenarios, such as the complexity of watermark key management, user-defined generation parameters, and the difficulty of verification by arbitrary third parties. To address this issue, we propose Gaussian Shading++, a diffusion model watermarking method tailored for real-world deployment. We propose a double-channel design that leverages pseudorandom error-correcting codes to encode the random seed required for watermark pseudorandomization, achieving performance-lossless watermarking under a fixed watermark key and overcoming key management challenges. Additionally, we model the distortions introduced during generation and inversion as an additive white Gaussian noise channel and employ a novel soft decision decoding strategy during extraction, ensuring strong robustness even when generation parameters vary. To enable third-party verification, we incorporate public key signatures, which provide a certain level of resistance against forgery attacks even when model inversion capabilities are fully disclosed. Extensive experiments demonstrate that Gaussian Shading++ not only maintains performance losslessness but also outperforms existing methods in terms of robustness, making it a more practical solution for real-world deployment.

\end{abstract}

\section{Introduction}
\IEEEPARstart{D}{}iffusion models~\cite{sohl2015deep,song2019generative,song2020score,ho2020denoising,song2020denoising} signify a noteworthy leap forward in image generation. These well-trained diffusion models, especially commercial models like Stable Diffusion (SD)~\cite{rombach2022high}, Glide~\cite{nichol2021glide}, and Muse AI~\cite{rombach2022high}, enable individuals with diverse backgrounds to create high-quality images effortlessly. However, this raises concerns about intellectual property and whether diffusion models will be stolen or resold twice. 

On the other hand, the ease of generating realistic images raises concerns about potentially misleading content generation. For example, on May 23, 2023, a Twitter-verified user named Bloomberg Feed posted a tweet titled ``Large explosion near the Pentagon complex in Washington DC-initial report," along with a synthetic image. This tweet led to multiple authoritative media accounts sharing it, even causing a brief impact on the stock market\footnote{\href{https://www.cnn.com/2023/05/22/tech/twitter-fake-image-pentagon-explosion/index.html}{Fake image of Pentagon explosion on Twitter.}}. On June 13, 2024, the European Union enacted the Artificial Intelligence Act, which mandates the implementation of technical safeguards including watermarks to prevent AI-generated content from misleading the public, ensuring transparency and credibility in the information ecosystem.\footnote{\href{https://eur-lex.europa.eu/legal-content/EN/TXT/?uri=CELEX:32024R1689}{Artificial Intelligence Act: Regulation (EU) 2024/1689 of the European Parliament and of the Council.}}. The urgency of labeling generated content for copyright authentication and prevention of misuse is evident.

 \begin{figure}[t]
  \centering
\includegraphics[width=\linewidth]{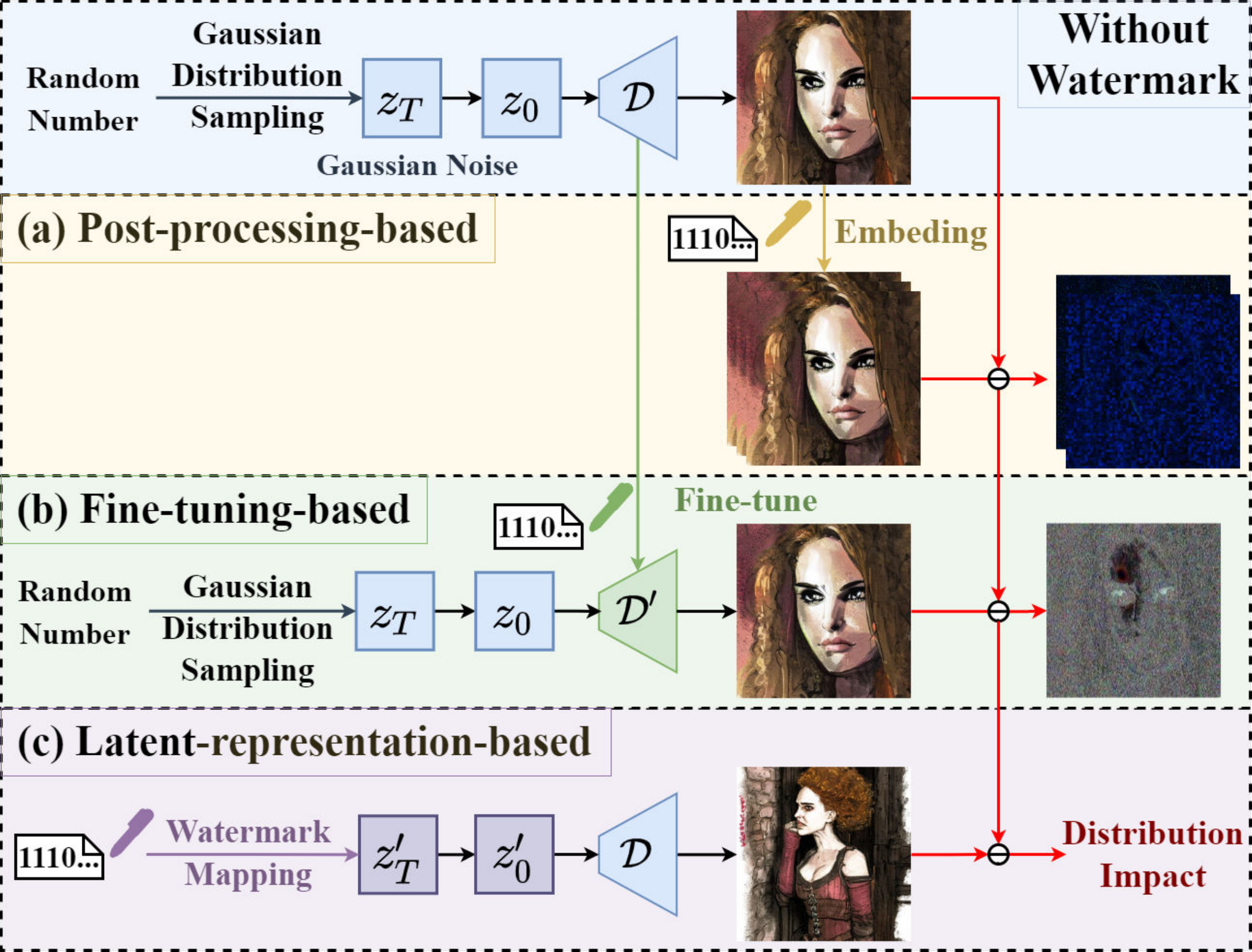}
   \caption{Existing watermarking frameworks can be divided into three categories: (a) post-processing-based, (b) fine-tuning-based, and (c) latent-representation-based. Since methods (a) and (b) either introduce watermark residuals or require additional computational overhead, method (c) has emerged as the mainstream approach by overcoming these two drawbacks. Their performance is primarily evaluated based on the impact on the distribution. }
   \label{fig:watermark}
\end{figure}

Watermarking is highlighted as a fundamental method for labeling generated content, as it embeds watermark information within the generated image, allowing for subsequent copyright authentication and the tracking of false content. Existing watermarking methods for the diffusion model can be divided into three categories, as shown in Fig.~\ref{fig:watermark}. Post-processing-based methods~\cite{cox2007digital,zhang2019robust,zhu2018hidden,luo2020distortion,zhang2020udh,kishore2021fixed, tancik2020stegastamp,zhong2020automated,jia2021mbrs,ma2022towards,fernandez2022watermarking,fang2022pimog,fang2023flow} adjust robust image features to embed watermarks, thereby directly altering the image and degrading its quality. 
To mitigate this concern, recent research endeavors propose fine-tuning-based methods~\cite{fernandez2023stable,zhao2023recipe,liu2023watermarking,cui2023diffusionshield,xiong2023flexible}, which amalgamate the watermark embedding process with the image generation process. Intuitively, these methods need to modify model parameters, introducing supplementary computational overhead. 

To address two main concerns of visible watermark residuals in generated images and excessive computational overhead inherent in prior methods,  latent-representation-based methods~\cite{wen2023tree, gs,prc} has emerged as a promising solution. These methods ensure that watermark information remains invisible in the image and offer plug-and-play functionality without requiring training, making it an increasingly important focus of research. Wen et al.~\cite{wen2023tree} introduced the Tree-Ring Watermark (TRW), the first of its kind, which embeds information by modifying latent representations to align with specific patterns. However, it restricts the randomness of sampling, which impacts the generative performance. Our earlier work, Gaussian Shading~\cite{gs}, resolves this limitation by incorporating pseudorandom keys and distribution-preserving sampling during the mapping of watermark information to latent representations. This ensures that the distribution of watermarked images matches that of non-watermarked images, achieving the first provable performance-lossless watermarking scheme. However, Gaussian Shading~\cite{gs} requires assigning a unique pseudorandom key to each image, which introduces significant challenges in key management in practical implementation. To tackle the key management issue, Gunn et al.~\cite{prc} proposed an undetectable watermarking method called the pseudorandom error-correcting codes watermark (PRCW). The core of PRCW~\cite{prc} lies in pseudorandom error-correcting codes (PRC)~\cite{christ2024pseudorandom}, where the generator matrix and parity check matrix serve as the watermark key. Even when the watermark key is fixed, PRCW~\cite{prc} can generate pseudorandom bitstreams, which are then mapped to latent representations that follow a standard normal distribution. This effectively resolves the key management problem. 

However, our experiments reveal that the robustness of PRCW~\cite{prc} is highly sensitive to the \textit{ guidance\_scale} parameter. As shown in Tab.~\ref{tab:com_gs}, PRCW~\cite{prc} only demonstrates strong robustness when the \textit{guidance\_scale} values used during generation and inversion are identical. Performance degrades significantly when there is a mismatch between these values. This poses a challenge for real-world deployment, as platforms typically allow users to customize generation parameters, which are unknown during watermark verification.

Based on the above analysis, latent-representation-based watermarking methods currently face three main challenges in practical deployment: First, how to achieve performance-lossless watermarking, where the distribution of watermarked images aligns with that of non-watermarked images, with a fixed watermark key to simplify key management. Second, watermarking schemes should accommodate user-customized generation parameters and maintain robust performance even when mismatches occur between the generation and verification phases. Lastly, many existing watermarking schemes depend on operators for both watermark embedding and verification, introducing potential security and trustworthiness risks in the authentication process~\cite{fairoze2025difficulty}. Consequently, watermark verification needs to be publicly accessible to any third party.

To address the aforementioned three challenges, we propose Gaussian Shading++,  a more practical performance-lossless watermarking scheme designed for real-world deployment scenarios. 
Specifically, during watermark embedding, we propose a double-channel design, in which the latent space is evenly divided along the channel dimension into two parts: the GS Channel and the PRC Channel. The GS Channel continues to use Gaussian Shading~\cite{gs}, where the robustness is enhanced by replicating the watermark content. The PRC Channel serves as the watermark header, encoding the seed for the pseudorandom number generator (PRNG) used in Gaussian Shading~\cite{gs} via PRC~\cite{christ2024pseudorandom}. The actual watermark key is the generator matrix and parity check matrix of PRC, and a private key that is used together with the seed to drive the PRNG. While the watermark key remains fixed, the seed can vary freely. The fixed watermark key does not compromise the pseudorandomness of the PRC Channel, and the randomly sampled seed ensures the pseudorandomness of the GS Channel. Thus, the entire watermark remains pseudorandom even with a fixed watermark key. Under the effect of distribution-preserving sampling, both the latent representations mapped by the watermark and those from random sampling follow a standard Gaussian distribution. Since the subsequent generation process remains unchanged, the distribution of watermarked images aligns with that of non-watermarked images, achieving performance-lossless watermarking while addressing the key management issue.

During watermark extraction, inspired by PRCW~\cite{prc}, we employ a more accurate Exact Inversion method~\cite{Hong_2024_CVPR} to recover the latent representations from images. The seed is then extracted from the PRC Channel. Next, we propose a novel soft decision decoding strategy for the GS Channel, which significantly improves the hard decision method in Gaussian Shading~\cite{gs} that employs direct binarization of latent representations. By modeling the entire generation and inversion process as an additive white Gaussian noise channel (AWGN Channel), we reformulate watermark recovery as a maximum likelihood decoding problem of repetition codes (REP codes)~\cite{ryan2009channel}.
Specifically, each watermark bit undergoes multiple embeddings in the latent space, and we compute the posterior expectation of each repeated instance based on the estimated noise. These expectations are then aggregated to form a log-likelihood ratio (LLR) for each bit, enabling near-optimal soft decision decoding that equivalently performs maximum a posteriori (MAP) estimation under the assumption of independent Gaussian noise.
Thanks to the accurate inversion, efficient modeling, and soft decision decoding, the robustness of Gaussian Shading++ is further improved, ensuring that the watermark can still be correctly extracted even when the parameter \textit{guidance\_scale} varies.

To further extend the scheme to third-party verifiability, it is necessary to make the model inversion capability public, which implies that third parties can access the latent representations, thereby introducing the risk of forgery using surrogate models~\cite{müller2024blackbox}. To address this, we incorporate the public-key signature ECDSA~\cite{ECDSA} into the framework, using a private signing key to sign the user information in the GS Channel. During extraction, the extracted watermark from the GS Channel is further verified using a public verification key to authenticate the signature. After forgery, the watermark accuracy of the image decreases, making it difficult to pass the signature verification. This ensures that Gaussian Shading++ can resist a certain degree of forgery attacks~\cite{müller2024blackbox}.

To demonstrate the effectiveness of our method, we evaluate Gaussian Shading++ against traditional distortions and neural network-based removal attacks~\cite{balle2018variational,cheng2020learned,esser2021taming,rombach2022high}, comparing it with several state-of-the-art methods to validate its superior robustness. Additionally, we highlight the performance-lossless nature of Gaussian Shading++ by comparing both the visual quality of images and the distribution of latent representations. Furthermore, Gaussian Shading++ exhibits strong robustness even in scenarios where the \textit{guidance\_scale} parameter varied during generation. Finally, we assess the robustness of Gaussian Shading++ in third-party verifiable scenarios, demonstrating a certain level of resilience against existing forgery attacks~\cite{müller2024blackbox}. Overall, Gaussian Shading++ demonstrates strong practicality in real-world applications.

 To summarize, our contributions are as follows:
 
 \begin{itemize}
    \item We propose a double-channel design, utilizing the PRC Channel to encode the random seed required for the pseudorandomization of the GS Channel. This approach overcomes the limitations of complex key management in Gaussian Shading and achieves performance-lossless even with a fixed watermark key.

    \item By building on the prior modeling of the generation and inversion process as the AWGN Channel, we propose a novel soft decision decoding strategy for the maximum likelihood decoding in REP codes. This near-optimal approach effectively enhances the robustness of the GS Channel, enabling Gaussian Shading++ to achieve excellent performance across varying generation parameters.

    \item By introducing the public-key signature ECDSA, we extend the application scenario to enable verification by any third party, while providing a certain level of resistance against existing forgery attacks.

    \item The experimental results demonstrate that our method outperforms state-of-the-art methods in terms of both robustness and performance losslessness, further advancing the practical applicability of watermarking in diffusion models.

 \end{itemize}
\section{Related Work}

\subsection{Diffusion Models}

Inspired by nonequilibrium thermodynamics~\cite{sohl2015deep}, Ho et al.~\cite{ho2020denoising} introduced the Denoising Diffusion Probabilistic Model (DDPM). DDPM consists of two Markov chains used for adding and removing noise, and subsequent works~\cite{song2020denoising,dhariwal2021diffusion, nichol2021improved,lu2022dpm,gu2022vector,ho2022classifier,rombach2022high, podell2023sdxl} have adopted this bidirectional chain framework. To reduce computational complexity and improve efficiency, the Latent Diffusion Model (LDM)~\cite{rombach2022high} was designed, in which the diffusion process occurs in a latent space $\mathcal{Z}$. To map an image $x \in \mathbb{R}^{H\times W\times 3}$ to the latent space, the LDM employs an encoder $\mathcal{E}$, such that $z_0 = \mathcal{E}(x) \in \mathbb{R}^{h\times w\times ch}$. Similarly, to reconstruct an image from the latent space, a decoder $\mathcal{D}$ is used, such that $x = \mathcal{D}(z_0)$. A pretrained LDM can generate images without the encoder $\mathcal{E}$. Specifically, a latent representation $z_T$ is first sampled from a standard Gaussian distribution $ \mathcal{N}(0,I)$. Subsequently, through iterative denoising using methods like DDIM~\cite{song2020denoising}, $z_0$ is obtained, and an image can be generated using the decoder: $x = \mathcal{D}(z_0)$.

\subsection{Diffusion Inversion}
Diffusion inversion~\cite{song2020denoising,wallace2023edict,zhang2024exact, Hong_2024_CVPR} can be regarded as the inverse process of generation, recovering latent representations from images for various downstream tasks, such as image editing~\cite{hertz2022prompt,mokady2023null,zhang2023inversion,mo2024freecontrol,li2024self,chung2024style} and watermark detection~\cite{wen2023tree,gs,prc}. The most native approach is DDIM Inversion~\cite{song2020denoising}, which simply reverses the time axis and employs the same sampling method as DDIM generation. However, DDIM Inversion struggles with stable reconstruction of real images, potentially leading to incorrect image reconstruction in downstream tasks. Wallace et al.~\cite{wallace2023edict} proposed maintaining two coupled latent representations and achieving precise inversion of real and generated images through an alternating approach. Zhang et al.~\cite{zhang2024exact} introduced a bi-directional integration approximation method to perform exact diffusion inversion. However, while improving inversion accuracy, these methods modify the sampling process and are not universally applicable to images generated by common sampling methods~\cite{song2020denoising,lu2022dpm}. Hong et al.~\cite{Hong_2024_CVPR} achieved a more general diffusion inversion method by utilizing gradient descent or forward step methods, further enhancing reconstruction accuracy.

\subsection{Image Watermark}

Digital watermark~\cite{van1994digital} is an effective means to address copyright protection and content authentication by embedding copyright or traceable identification information within carrier data. Typically, the functionality of a watermark depends on its capacity. For example, a single-bit watermark can determine whether an image was generated by a particular diffusion model, i.e., copyright protection; a multi-bit watermark can further determine which user of the diffusion model generated the image, i.e., traceability.

Image watermark is a method that employs images as carriers for the watermark. Initially, watermark embedding methods primarily focused on the spatial domain~\cite{van1994digital}, but later, to enhance robustness, transform domain watermarking techniques~\cite{kundur1997robust,tsai2000joint,guo2002digital,lee2007reversible,al2007combined,stankovic2009application,hamidi2018hybrid} were developed. In recent years, with the advancement of deep learning, researchers have turned their attention to neural networks~\cite{lecun1998gradient,vaswani2017attention}, harnessing their powerful learning capabilities to develop watermarking techniques~\cite{zhu2018hidden,luo2020distortion,zhang2020udh,kishore2021fixed, tancik2020stegastamp,zhong2020automated,jia2021mbrs,ma2022towards,fernandez2022watermarking,fang2022pimog,fang2023flow}.

\subsection{Image Watermark for Diffusion Models}

Existing image watermarking methods for the diffusion model~\cite{cox2007digital,zhang2019robust,zhu2018hidden,luo2020distortion,zhang2020udh,kishore2021fixed, tancik2020stegastamp,zhong2020automated,jia2021mbrs,ma2022towards,fernandez2022watermarking,fang2022pimog,fang2023flow,fernandez2023stable,zhao2023recipe,liu2023watermarking,cui2023diffusionshield,xiong2023flexible,wen2023tree,gs,prc} can be divided into three categories, as shown in Fig.~\ref{fig:watermark}. The image watermarking methods described in the previous section can be applied directly to the images generated by the diffusion model, which is called post-processing-based watermarks~\cite{cox2007digital,zhang2019robust,zhu2018hidden,luo2020distortion,zhang2020udh,kishore2021fixed, tancik2020stegastamp,zhong2020automated,jia2021mbrs,ma2022towards,fernandez2022watermarking,fang2022pimog,fang2023flow}. These methods directly modify the image, thus degrading image quality. Recent research endeavors have amalgamated the watermark embedding process with the image generation process to mitigate this issue. Stable Signature~\cite{fernandez2023stable} fine-tunes the LDM decoder using a pre-trained watermark extractor, facilitating watermark extraction from images produced by the fine-tuned model. Zhao et al.~\cite{zhao2023recipe} and Liu et al.~\cite{liu2023watermarking} suggest fine-tuning the diffusion model to implant a backdoor as a watermark, enabling watermark extraction by triggering. These fine-tuning-based approaches enhance the quality of watermarked images but introduce supplementary computational overhead and modify model parameters. 

To address the limitations of the aforementioned two types of methods, Wen et al.~\cite{wen2023tree} proposed the first latent-representation-based method named the Tree-Ring Watermark (TRW), which conveys copyright information by adapting the frequency domain of latent representations to match specific patterns. This method achieves an imperceptible watermark. However, it directly disrupts the Gaussian distribution of noise, limiting the randomness of sampling and resulting in affecting model performance. In our previous work, Gaussian Shading~\cite{gs}, we introduced a stream cipher and distribution-preserving sampling to ensure that the distribution of watermarked images matches that of non-watermarked images, achieving a performance-lossless watermarking scheme. However, Gaussian Shading requires assigning a unique stream key to each image, leading to key management challenges. Gunn et al.~\cite{prc} proposed the pseudorandom error-correcting codes watermark (PRCW), which uses pseudorandom error-correcting codes~\cite{christ2024pseudorandom} to encode watermark information. This approach allows the generation of latent representations that conform to a standard Gaussian distribution even when the watermark key is fixed, effectively addressing the key management issue in the performance-lossless watermark. However, the robustness of PRCW degrades when faced with variations in generation parameters.

In this paper, we aim to design a watermarking method that achieves performance-lossless even with a fixed watermark key, while maintaining strong robustness under scenarios with varying generation parameters. Additionally, considering practical applicability, the method should ideally be extendable to support verification by any third party.

\subsection{Pseudorandom Error-correcting Codes Based on LDPC Codes} \label{sec:ldpc-prc}
To address the key reuse issue in Gaussian Shading~\cite{gs} and enable consistent performance-lossless capabilities across multiple image generations, we introduce the pseudorandom error-correcting codes (PRC) based on LDPC codes~\cite{christ2024pseudorandom} to embed a header composed of random seeds.

The key generation, encoding, and decoding procedures of the PRC based on LDPC codes are as follows:

\begin{itemize} 
    \item \texttt{KeyGen}: $(n, g, t, r) \mapsto (P, G) $
    \begin{itemize} 
        \item Sample a random matrix $P \in \mathbb{F}^{r \times n}_2$ (parity-check matrix) subject to every row of $P$ being $t$-sparse.
        \item Sample a random matrix $G \in \mathbb{F}^{n \times g}_2$ (generator matrix) subject to $PG=0$.
    \end{itemize}
    \item \texttt{Encode}$(m)$: Given message $m \in \mathbb{F}_2^g$, sample noise $e \leftarrow \operatorname{Ber}(n, \eta)$, and output ciphertext $c = Gm \oplus e$.
    \item \texttt{Decode}$(L)$: Given a vector of posterior soft information $L = \left( \ell_1, \ell_2, \dots, \ell_n \right)$ (e.g., $\ell_i = \mathbb{E}[m_i\mid c']$), apply the BP-OSD decoder to recover the original message $\hat{m}$.
\end{itemize}

The above definition corresponds to the standard regular LDPC codes, where the sparsity parameter $t$ is fixed. To satisfy the requirements of PRC, the following constraint must be imposed:

\begin{itemize} 
    \item The sparsity is set as $t = \Theta(\log n)$, and each execution of \texttt{Encode} samples a fresh message $m$ uniformly at random from $\mathbb{F}_2^g$.
\end{itemize}

When the above constraints are satisfied, the output of \texttt{Encode} is pseudorandom under either \textbf{the subexponential Learning Parity with Noise (LPN) assumption}, or under \textbf{the standard LPN assumption} combined with \textbf{the planted XOR assumption} (see in \cite{christ2024pseudorandom}); that is: For any polynomial-time adversary $\mathcal{A}$,
\begin{equation}
    \left|\underset{P,G}{\operatorname{Pr}}\left[\mathcal{A}^{\text {Encode}(\cdot)}\left(1^{\lambda}\right)=1\right]-\underset{\mathcal{U}}{\operatorname{Pr}}\left[\mathcal{A}^{\mathcal{U}}\left(1^{\lambda}\right)=1\right]\right| \leq \operatorname{negl}(\lambda).
\end{equation}

In later sections, we will show that by choosing an appropriate sparsity parameter $t$ and using a randomly sampled seed from $\mathbb{F}_2^{g}$ as the input to the PRC encode process, each image generation involves fresh randomness. As a result, our construction inherits the pseudorandomness guarantees established in prior theoretical work.
\section{Threat Model}

In this section, we introduce the threat model considered for the proposed method. As illustrated in Fig.~\ref{fig:scena}, it is divided into two scenarios: Operator Verification and Third-party Verification. We also present the Watermark Statistical Test for detection and traceability tasks.

\begin{figure*}[t]
  \centering
   \includegraphics[width=\linewidth]{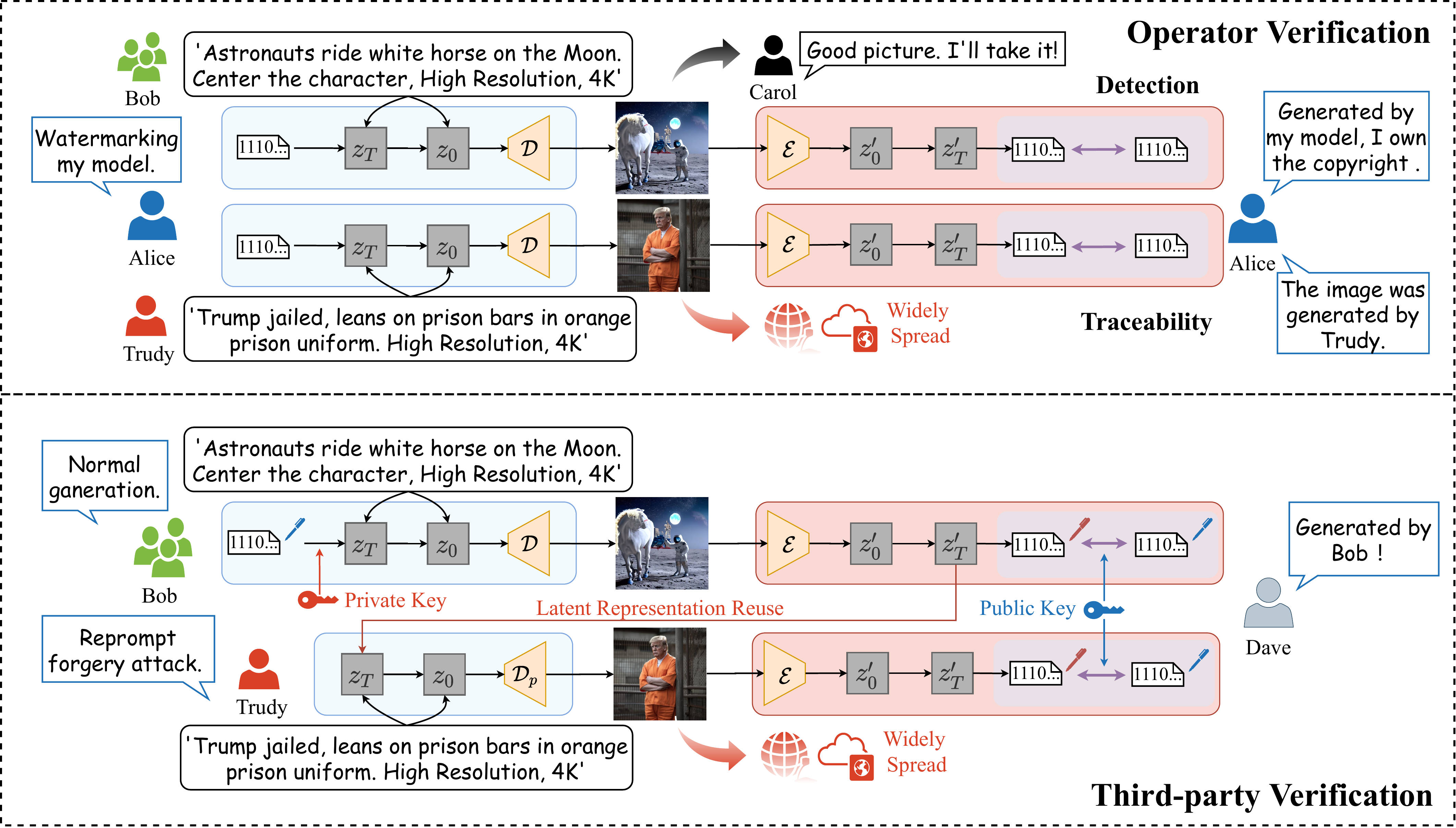}
   \caption{
   The two application scenarios of Gaussian Shading++ are Operator Verification and Third-party Verification. In the Operator Verification scenario, Gaussian Shading++ considers satisfying the requirements of generated image detection (copyright protection) and malicious user traceability. In the Third-party Verification scenario, Gaussian Shading++ considers the need for any third party to verify the watermark, and it aims to defend against the reprompt forgery attack that may emerge once model inversion capabilities become publicly available.}
   \label{fig:scena}
\end{figure*}

\subsection{Operator Verification}

As shown in Fig.~\ref{fig:scena}, the scenario involves the operator Alice, the thief Carol, and two types of users Bob and Trudy. 
\subsubsection{Operator Alice} Alice is responsible for training the model, deploying it on the platform, and providing the corresponding API for users, but she does not open-source the code or model weights. On one hand, to fulfill the detection (copyright protection) requirement, Alice embeds a single-bit watermark into each generated image. The successful extraction of the watermark from an image serves as evidence of Alice's rightful ownership of the copyright, while also indicating that the image is artificially generated (as opposed to natural images). On the other hand, to meet the traceability requirement, Alice assigns a unique watermark to each user. By extracting the watermark from illicit content, it becomes possible to trace malicious user Trudy through comparison with the watermark database. Traceability represents a higher-level objective than detection and can also achieve copyright protection for different users.

\subsubsection{Community user Bob} Bob faithfully adheres to the community guidelines, utilizing the API provided by Alice to generate and disseminate images.

\subsubsection{Thief Carol} Carol does not use Alice's services but steals images generated by her model, claiming ownership of the copyrights.

\subsubsection{Malicious user Trudy} Trudy uses the API provided by Alice to generate deepfakes and infringe content. To evade detection and traceability, Trudy can employ various data augmentation techniques to modify illicit images.

\subsection{Third-party Verification}

As shown in Fig.~\ref{fig:scena}, the scenario involves the operator Alice, any third party Dave who wishes to verify the watermark, and two types of users Bob and Trudy.

\subsubsection{Operator Alice} Alice provides users with an API, but in the Third-party Verification scenario, she must expose the model's inversion capability. To mitigate the risk of malicious user Trudy exploiting the watermark for forgery attacks, Alice can maintain a public-private key pair. During the generation process, Alice uses the private signing key to sign the user information, and this signature, combined with the user information, forms the watermark information. When a third party Dave, needs to verify the watermark, Alice can provide the public verification key directly or store it in a digital certificate, making it readily available for Dave's use. Additionally, to enable Dave to trace back to the target user based on the user information after signature verification, Alice also needs to maintain a public watermark database to facilitate Dave's query operations.

 \subsubsection{Any third party Dave} Dave can be any third party seeking to verify the watermark, and he requires both Alice's public verification key and access to the publicly available model inversion capability. Once the signature verification is successful, Dave can trace the user information in the watermark database maintained by Alice.

\subsubsection{Community user Bob} Bob faithfully adheres to the community guidelines, utilizing the API provided by Alice to generate and disseminate images. Bob's user information is signed using Alice's private signing key, and his user information is made publicly available in the watermark database.

\subsubsection{Malicious user Trudy} Trudy's goal is to disguise illicit content as being generated by other users. Although Alice has made the watermark database publicly accessible, Trudy cannot directly forge content without access to Alice's private key. Therefore, Trudy can only exploit the publicly available model inversion capability to obtain the target user's latent representations. By leveraging a proxy model $\mathcal{D}_p$, Trudy can reuse these latent representations to regenerate content under illicit prompts, thereby executing a reprompt forgery attack~\cite{müller2024blackbox}.

\subsection{Watermark Statistical Test}

\subsubsection{Detection} Alice embeds a single-bit watermark, represented by $q$-bit binary watermark $s \in \{0,1\}^q$, into each generated image using Gaussian Shading++. This watermark serves as an identifier for her model. Assuming that the watermark $s'$ is extracted from the image $X$, the detection test for the watermark can be represented by the number of matching bits between two watermark sequences, $Acc(s,s')$. When the threshold $\tau \in \{0,\dots,q\}$ is determined, if $Acc(s,s') \geq \tau$, it is deemed that $X$ contains the watermark.

In previous works~\cite{yu2021artificial}, it is commonly assumed that the extracted watermark bits $s'_1,\dots,s'_q$ from the vanilla images are independently and identically distributed, with $s'_i$ following a Bernoulli distribution with parameter $0.5$. Thus, $Acc(s,s')$ follows a binomial distribution $\operatorname{Ber}(q,0.5)$. 

Once the distribution of $Acc(s,s')$ is determined, the false positive rate ($\operatorname{FPR}$) is defined as the probability that $Acc(s,s')$ of a vanilla image exceeds the threshold $\tau$. This probability can be further expressed using the regularized incomplete beta function $B_x(a;b)$~\cite{fernandez2023stable},
\begin{equation}
    \begin{aligned}
    \operatorname{FPR}(\tau) &=\operatorname{Pr}\left(Acc\left(s, s^{\prime}\right)>\tau \right) =\frac{1}{2^{q}}\sum_{i=\tau+1}^{q} \left(\begin{array}{c}q \\i\end{array}\right)\\
    &=B_{1 / 2}(\tau+1, q-\tau) .\label{smeq:detection}
\end{aligned}
\end{equation}

\subsubsection{Traceability} To enable traceability, Alice needs to assign a watermark $s^i \in \{0,1\}^q$ to each user, where $i=1,\dots,N$ and $N$ represents the number of users. During the traceability test, the bit matching count $Acc(s^1,s'),\dots, Acc(s^N,s')$ needs to be computed for all $N$ watermarks. If none of the $N$ tests exceed the threshold $\tau$, the image is considered not generated by Alice's model. However, if at least one test passes, the image is deemed to be generated by Alice's model, and the index with the maximum matching count is traced back to the corresponding user, i.e., $\operatorname*{argmax}_{i=1,\dots, N}Acc(s^i,s')$. When a threshold $\tau$ is given, the $\operatorname{FPR}$ can be expressed as follows~\cite{fernandez2023stable},
\begin{equation}
    \operatorname{FPR}(\tau, N) = 1 - (1-\operatorname{FPR}(\tau))^N \approx N \cdot \operatorname{FPR}(\tau) . \label{smeq:trace}
\end{equation}

\section{Proposed Method}
In this section, we first introduce the workflow of the proposed method, Gaussian Shading++, as illustrated in Fig.~\ref{fig:framework}. We propose a double-channel design by partitioning the latent space into two components: the PRC Channel and the GS Channel. In the watermark key generation phase, a ternary key set needs to be generated for the PRC channel. In the Third-party Verification scenario, we also need to introduce the public-key signature ECDSA~\cite{ECDSA} key pair to defend against forgery attacks~\cite{müller2024blackbox}. During embedding, the PRC Channel serves as the watermark header, encoding the random seed required for the GS Channel. The GS Channel enhances robustness by performing diffusion on the actual watermark message. The two Channels are then merged and used to drive distribution-preserving sampling, followed by denoising to generate watermarked images. During extraction, Exact Inversion~\cite{Hong_2024_CVPR} is employed to recover the latent representation. The distortion throughout the entire generation and inversion process is modeled as an AWGN Channel, enabling posterior estimation of the latent representation symbols. Subsequently, the random seed in the PRC Channel is first recovered, followed by decryption and soft decision decoding to extract the watermark from the GS Channel.  At the end of this section, we provide theoretical proof of the performance-lossless characteristic of Gaussian Shading++.

\begin{figure*}[t]
  \centering
   \includegraphics[width=\linewidth]{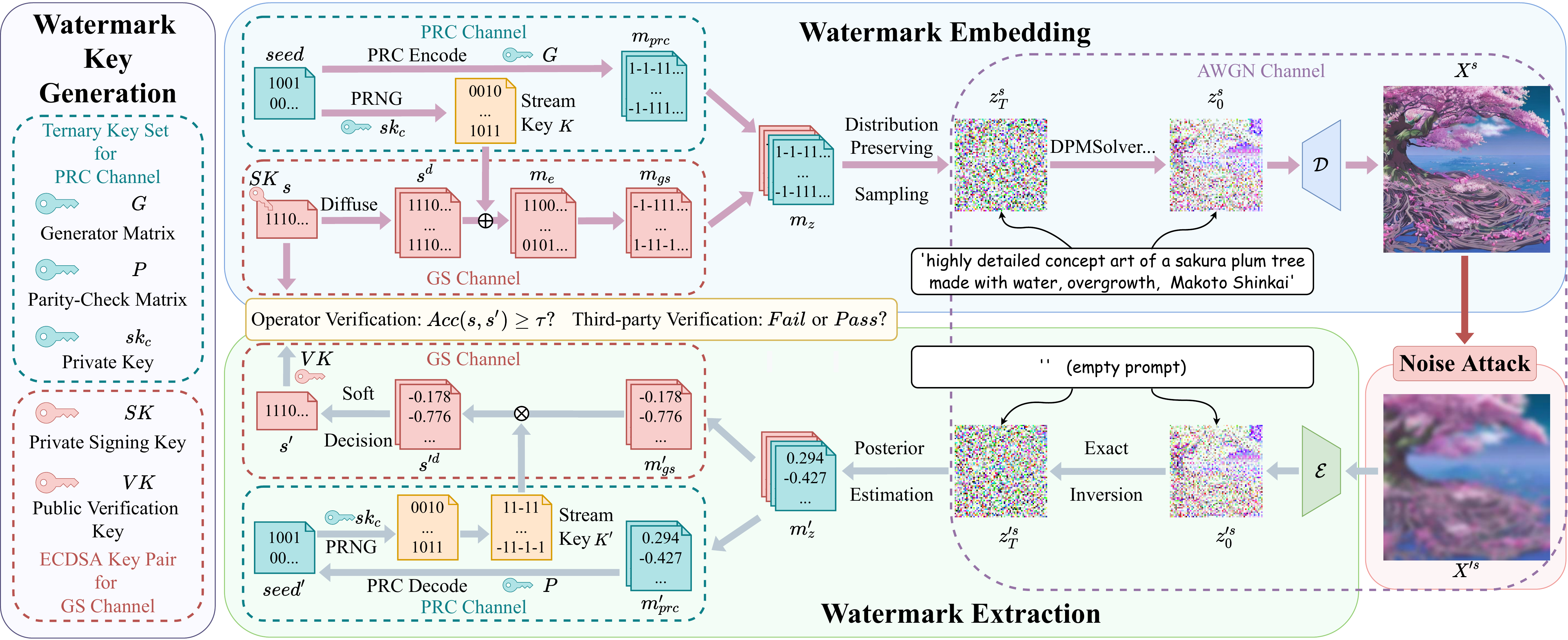}
   \caption{The framework of  Gaussian Shading++. The latent space is divided into the PRC Channel and GS Channel. During the watermark key generation, a ternary key set is generated for the PRC Channel. In the Third-party Verification scenario, a public-key signature ECDSA~\cite{ECDSA} key pair is introduced. During the watermark embedding, the PRC Channel serves as the header, encoding the random $seed$ to drive a PRNG, resulting in $m_{prc}$
 . The GS Channel embeds a $k$-bit watermark sequence $s$, which undergoes diffusion, encryption, and transformation to produce $m_{gs}$. The combined sequence of $m_{prc}$ and $m_{gs}$
  is used to drive distribution-preserving sampling, followed by denoising to generate watermarked images $X^s$. For watermark extraction, the process begins with Exact Inversion~\cite{Hong_2024_CVPR} to recover $z'_{T}$. The distortion throughout the entire generation and inversion process is modeled as an AWGN channel, enabling posterior estimation of the symbols of $z'_{T}$. Subsequently, the PRC Channel is first decoded to retrieve the random $seed$, which generates $K'$. $K'$ is employed to decrypt the GS Channel, and the final watermark is obtained through soft-decision decoding. In the Third-party Verification scenario, $s$ includes the user information and its signature, which must be verified after extracting $s'$.}
   \label{fig:framework}
\end{figure*} 

\subsection{Watermark Key Generation}
The watermarking scheme utilizes a composite key consisting of two components. The first component is an LDPC key used to construct the PRC Channel. The second component is a private key \( sk_c \), which is involved in the generation of the stream key for the GS Channel.

Suppose the latent representations have dimensions ${ch \times h \times w}$ and each dimension can represent $v$ bits. We designate the first half of the channels for PRC embedding, and refer to these channels as the PRC Channel. For the second half of the channels, we refer to them as the GS Channel.

We follow the framework of the pseudorandom error-correcting codes based on LDPC codes (Sec.~\ref{sec:ldpc-prc}). 
Specifically, the LDPC key is $(P,G)\leftarrow \texttt{KeyGen} \text{($\frac{ch \times h \times w}{2}, n_{seed},t,r$)}$. The private key \( sk_c \) is drawn uniformly at random from the binary space \( \{0,1\}^{n_{sk}} \).

For the Operator Verification scenario, the watermark key components \((P, G, sk_c)\) are treated as private.  
In contrast, for the Third-party Verification scenario, \((P, G, sk_c)\) are considered public parameters, while the actual watermark generation and extraction rely on an ECDSA~\cite{ECDSA} key pair \((SK, VK)\), where \(SK\) is the private signing key and \(VK\) is the corresponding public verification key.

\subsection{Watermark Embedding}

In both scenarios, watermark embedding should be strictly performed by the operator to prevent high-precision forgery attacks from malicious third parties. The specific process is as follows:
 
\subsubsection{PRC Channel} We first sample $seed \in \mathbb{F}_2^{n_{seed}}$ of PRNG uniformly at random, which is used to initialize the PRNG within the GS Channel for stream cipher generation. Then, we encode $seed$ into a codeword using the LDPC codes-based PRC: sample noise $e \leftarrow \operatorname{Ber}(n, \eta)$, $m_{c} = G\cdot seed \oplus e$.
Finally, a mapping is applied to convert $0$ to $1$ and $1$ to $-1$ on the  $m_c$, resulting in $m_{prc} = (-1)^{m_c}$. 

\subsubsection{GS Channel}

Given that the entire latent space is equally partitioned into two components, the watermark capacity becomes $\frac{{v\times ch\times h\times w}}{2}$ bits. To enhance the robustness of the watermark, we introduce a spatial replication factor $\frac{1}{f_{hw}}$ and a channel replication factor $\frac{1}{f_{ch}}$. We represent the watermark using $\frac{1}{f_{hw}}$ of the height and width, and $\frac{1}{f_{ch}}$ of the channel, and replicate the watermark $f_{ch} \cdot f^2_{hw}$ times. Thus, the watermark $s$ with dimensions $v\times \frac{ch}{2f_{ch}}\times \frac{h}{f_{hw}}\times \frac{w}{f_{hw}}$  is expanded into a diffused watermark $s^d$ with dimensions $\frac{{v\times ch\times h\times w}}{2}$. The actual watermark capacity is $q=\frac{{v\times ch\times h\times w}}{2f_{ch} \cdot f^2_{hw}}$ bits. In the third-party verifiable scenario, to defend against forgery attacks~\cite{müller2024blackbox}, we further introduce the public-key signature ECDSA~\cite{ECDSA}, where the watermark $s$ consists of the user information and its signature generated using the private signing key $SK$.

If we know the distribution of the diffused watermark $s^d$, we can directly utilize distribution-preserving sampling to obtain the corresponding latent representations $z^s_T$. However, in practical scenarios, its distribution is always unknown. Hence, we introduce a stream key $K$ to transform $s^d$ into a distribution-known randomized watermark $m_e$ through encryption. Specifically, we derive $K$ by concatenating the seed embedded in the PRC Channel with the private key $sk_c$, i.e., $K = \text{PRNG}(\text{H}(seed || sk_c))$, where $\text{H}$ denotes a cryptographic hash function and PRNG is a pseudorandom number generator used to produce $K$. Considering the use of PRNG, $m_e$ follows a uniform distribution, i.e., $m_e$ is a random binary bit stream. To be compatible with the PRC Channel, a mapping from $0$ to $1$ and from $1$ to $-1$ is required. That is, $m_{gs}=(-1)^{m_e}$.

\subsubsection{Distribution-preserving sampling driven by randomized watermark.}
At this point, $m_{prc}$ and $m_{gs}$ are combined into the watermark $m_z$ according to the channel where $m_{prc}$ is placed in the first two channels of the latent space, and $m_{gs}$ is placed in the last two channels.

Since both  $m_{prc}$ and $m_{gs}$ are pseudorandom, $m_z$ is also pseudorandom. We will prove this in Sec.~\ref{sec:prove}. Once the distribution of $m_z$ is known, a distribution-preserving sampling can be performed.

When each dimension represents $v$-bit randomized watermark $m_z$, these $v$ bits can be regarded as an integer $y \in [0, 2^v-1]$, where we treat $-1$ as binary $1$ and $1$ as binary $0$. Since $m_z$ is pseudorandom, $y$ follows a discrete uniform distribution, i.e., $p(y) = \frac{1}{2^v}$ for $y = 0, 1, 2, \dots, 2^v-1$. Let $f(x)$ denote the probability density function of the Gaussian distribution $\mathcal{N}(0, I)$, and $\text{ppf}(\cdot)$ denotes the quantile function. We divide $f(x)$ into $2^v$ equal cumulative probability portions. When $y=i$, the watermarked latent representation $z^s_T$ falls into the $i$-th interval, which means $z^s_T$ should follow the conditional distribution:
\begin{equation}
      p(z^s_T|y=i)=\left\{\begin{array}{cl}		2^v\cdot f(z^s_T) & \text{ppf}(\frac{i}{2^v}) < z^s_T \leq \text{ppf}(\frac{i+1}{2^v})  \\ 	0 &\text{otherwise}\end{array} \right. .\label{eq:conditional_dis}
\end{equation}
The probability distribution of $z^s_T$ is given by:
\begin{equation}
    p(z^s_T) = \sum^{2^v-1}_{i=0} p(z^s_T|y=i)p(y=i)=f(z^s_T) .\label{eq:distribution}
\end{equation}
Eq.~\ref{eq:distribution} indicates that $z^s_T$  follows the same distribution as the randomly sampled latent representation $z_T \sim \mathcal{N}(0, I)$. Next, we elaborate on how this sampling is implemented.

Let the cumulative distribution function of $\mathcal{N}(0, I)$ be denoted as $\text{cdf}(\cdot)$. We can obtain the cumulative distribution function of  Eq.~\ref{eq:conditional_dis} as follows,
\begin{equation}
\begin{aligned}
     &F(z^s_T|y=i)\\
     &\quad =\left \{\begin{array}{cl}	0&z^s_T<\text{ppf}(\frac{i}{2^v})\\	2^v\cdot \text{cdf}(z^s_T) - i & \text{ppf}(\frac{i}{2^v}) \leq z^s_T \leq \text{ppf}(\frac{i+1}{2^v})  \\ 	1 &z^s_T>\text{ppf}(\frac{i+1}{2^v}) \end{array} \right..\label{eq:cdf_condition}
\end{aligned}
\end{equation}

Given $y=i$, we aim to perform random sampling of $z^s_T$ within the interval $[\text{ppf}(\frac{i}{2^v}),\text{ppf}(\frac{i+1}{2^v})]$. 
The commonly used method is rejection sampling~\cite{hopper2002provably,chen2021distribution, add}, which can be time-consuming as it requires repeated sampling until $z^s_T$ falls into the correct interval. Instead, we can utilize the cumulative probability density. When randomly sampling $F(z^s_T|y=i)$, the corresponding $z^s_T$ is naturally obtained through random sampling. Since $F(z^s_T|y=i)$ takes values in $[0,1]$, sampling from it is equivalent to sampling from a standard uniform distribution, denoted as $u=F(z^s_T|y=i) \sim \mathcal{U}(0,1)$. Shift the terms of Eq.~\ref{eq:cdf_condition}, and take into account that cdf and ppf are inverse functions, we have
\begin{equation}
    z^s_T=\text{ppf}(\frac{u+i}{2^v}). \label{eq:z}
\end{equation}

Eq.~\ref{eq:z} represents the process of sampling the watermarked latent representation $z^s_T$ driven by the randomized watermark $m_z$.

\subsubsection{Image Generation}

After the sampling process, the watermark is embedded in the latent representation  $z^s_T$, and the subsequent generation process is no different from the regular generation process of SD. Here, we utilize the commonly adopted DPMSolver~\cite{lu2022dpm} for iterative denoising of $z^s_T$.
After obtaining denoised $z^s_0$, the watermarked image $X^s$ is generated using the decoder $\mathcal{D}$: $X^s=\mathcal{D}(z^s_0)$.

\subsection{Watermark Extraction}

In the Operator Verification scenario, watermark extraction is performed by the operator, and the verification result is made publicly available. In the Third-party Verification scenario, the operator is required to disclose the model inversion capability, allowing any third party to perform watermark extraction. The specific process is as follows:

\subsubsection{Inversion and Posterior Estimation }

Using the encoder $\mathcal{E}$, we first restore $X'^s$ to the latent space $z'^s_0 = \mathcal{E}(X'^s)$. Imprecise latent representation recovery can significantly reduce the effectiveness of watermark extraction. To address this, we employ the more precise Exact Inversion method~\cite{Hong_2024_CVPR} to estimate the additive noise.
After a sufficient number of inversion steps, $z'^s_T$ can be considered approximately equal to $z^s_T$ within an acceptable error margin.

After obtaining the inversion result $z'^s_T$, we follow the approach of Gunn et al.~\cite{christ2024pseudorandom} and model the entire image generation and inversion process as a noisy channel. Specifically, the concatenated codeword $m_z = m_{prc} || m_{gs} \in \{-1, 1\}^{ch \times h \times w}$ is viewed as passing through an AWGN Channel. The noise strength is characterized by $\sigma = \sqrt{3/2}$, and the posterior estimates $m'_z = m'_{prc} || m'_{gs}$ are obtained via the error function (erf) corresponding to the AWGN Channel, that is:
\begin{equation}
m'_z \triangleq \mathbb{E}\left[m \mid z'^s_T\right]=\operatorname{erf}\left(\frac{z'^s_T}{\sqrt{2 \sigma^{2}\left(1+\sigma^{2}\right)}}\right)
\end{equation}

After performing posterior estimation to obtain $m'_z$, we evenly split $m'_z$ along the channel dimension into two components: the PRC Channel watermark header $m'_{prc}$ and the GS Channel watermark information $m'_{gs}$, which are then processed separately.
 
\subsubsection{PRC Channel}
Assuming the LDPC key is given by $(P, G)$, we employ the belief propagation with ordered statistics decoding (BP-OSD)~\cite{bp,osd} to recover the estimated seed $n'_{seed}$ : $n'_{seed}=\text{BP-OSD}(P, m'_{prc})$. This recovered seed is then used to initialize the PRNG within the GS Channel for stream cipher regeneration during watermark extraction. 

\subsubsection{GS Channel}

For the GS Channel watermark information \( m'_{gs} \), decryption is performed via element-wise multiplication with the stream key \( K' = \text{PRNG}(\text{H}(seed'||sk_c)) \), since the following equation holds:
\begin{equation}
    s'^d\triangleq \mathbb{E}\left[(-1)^{s^d} \mid z'^s_T\right]=(-1)^{K'}\cdot m'_{gs},
\end{equation}
where \( s'^d \in [-1,1] \) is the posterior expectation of the repeated message \( (-1)^{s^d} \). To obtain the maximum a posteriori (MAP) decoding of the original message \( s \), the problem essentially reduces to decoding a repetition code (REP code~\cite{ryan2009channel}). 

For each bit \( s_j \) of the watermark information \( s \), the log-likelihood ratio (LLR) of its $i$-th repetition $s^d_{i,j}$ can be computed as follows:
\begin{equation}
    L L R_{i,j}=\log \frac{\operatorname{Pr}\left((-1)^{s^d_{i,j}}=1 \mid z'^s_T\right)}{\operatorname{Pr}\left((-1)^{s^d_{i,j}} = -1 \mid z'^s_T\right)}.
\end{equation}

Since $s'^d_{i,j}\triangleq \mathbb{E}\left[(-1)^{s^d_{i,j}} \mid z'^s_T\right]=\operatorname{Pr}\left((-1)^{s^d_{i,j}}=1 \mid z'^s_T\right)-\operatorname{Pr}\left((-1)^{s^d_{i,j}}=-1 \mid z'^s_T\right)$, it is easy to derive that:
\begin{equation}
    s'^d_{i,j} = \frac{e^{LLR_{i,j}}-1}{e^{LLR_{i,j}}+1} = \tanh(\frac{LLR_{i,j}}{2}).
\end{equation}

For each bit \( s_j \) of the watermark information \( s \),  
due to the independence of repeated observations, the total log-likelihood ratio \( LLR_{\text{total},j} \) for a repetition count of \(num= \frac{ch \times h \times w}{2f_{ch} \cdot f^2_{hw}} \) can be calculated by:
\begin{equation}
\begin{aligned}
    LLR_{\text{total}, j} 
    &\triangleq \log \frac{\operatorname{Pr}\left((-1)^{s_j} = 1 \mid z'^s_T\right)}{\operatorname{Pr}\left((-1)^{s_j} = -1 \mid z'^s_T\right)} \\
    &= \sum_{i=1}^{num} LLR_{i,j} 
    = \sum_{i=1}^{num} 2 \cdot \text{arctanh}(s'^d_{i,j}),
\end{aligned}
\end{equation}
Therefore, the final estimated \( s'_j \) of \( s_j \) can be calculated by:
\begin{equation}
    s'_j = \text{sign}(LLR_{\text{total},j}). \label{eq:full_llr}
\end{equation}

In practice, we use the first-order approximation of the arctanh function for computational efficiency, i.e., \( LLR_{i,j} \simeq 2s'^d_{i,j} \). The final estimated $s'_j$ is then given by:
\begin{equation}
    s'_j = \text{sign}
    (\sum_{i=1}^{num}2s'^d_{i,j})=\text{sign}
    (\sum_{i=1}^{num}s'^d_{i,j}).\label{eq:est_llr}
\end{equation}
The estimation method (Eq.~\ref{eq:est_llr}) shows negligible performance difference compared to the optimal scheme (Eq.~\ref{eq:full_llr}) that accumulates full LLRs.

In the Third-party Verification scenario, after the watermark is extracted, the user information and signature need to be verified using the public verification key $VK$.

\subsection{Proof of Lossless Performance} \label{sec:prove}

To demonstrate that our proposed hybrid watermarking scheme is performance-lossless, a sufficient condition is to prove that the embedded ciphertext $m_{prc}||m_{gs}$ satisfies the IND\$-CPA security property—i.e., it is computationally indistinguishable from a random bitstring under chosen-plaintext attacks. This suffices because the sampling process we adopt is distribution-preserving: as long as the ciphertext driving the sampling is IND\$-CPA secure, the resulting watermarked image remains computationally indistinguishable from a randomly sampled image under chosen-plaintext attacks~\cite{thietke2025towards}.

In the following, we present the formal definition of IND\$-CPA security, and subsequently provide a proof that the ciphertext $m_{prc}||m_{gs}$ used in our scheme satisfies this security notion.

\newtheorem{definition}{Definition}

\begin{definition}\label{IND-CPA} \textbf{(Chosen Plaintext Attack)}

Consider a symmetric encryption scheme with a key tuple \(K_{\mathsf{CS}} = (P, G, sk_c)\), which produces ciphertexts of the form:  
\[
(m_{\text{prc}} \,\|\, m_{\text{gs}})_{s^d} \triangleq (G \cdot seed \oplus e) \,\|\, \left( \mathsf{PRNG}(\mathsf{H}(seed \,\|\, sk_c)) \oplus s^d \right),
\]  
where \(s^d\) denotes the plaintext to be encrypted, and $\mathsf{H}:\{0,1\}^*\rightarrow \mathbb{F}_2^k$ is a hash function. The LDPC key pair \((P, G)\) is generated via the randomized algorithm \(\textbf{KeyGen}(l(k), k, t, l'(k))\) and $sk_c$ is randomly sampled with a length of $k$ bits, all parameters depend on the security parameter \(k\), except for the constant \(t\).

We define a chosen-plaintext attack (CPA) game between an adversary \(\mathcal{A}\) and a challenger as follows:

\begin{itemize}
    
    \item \textbf{Key generation stage.} $K_{\mathsf{CS}}=(P,G,sk_c)\leftarrow (\textbf{KeyGen}(l(k), k, t, l'(k)),1^k)$.

    \item \textbf{Learning stage.} $\mathcal{A}$ sends plaintext $s^d_\mathcal{A}$ to the oracle and  returns $(m_{\text{prc}} \,\|\, m_{\text{gs}})_{s^d_\mathcal{A}}$. $\mathcal{A}$ can perform this stage multiple times.
    \item \textbf{Challenge stage.} $\mathcal{A}$ sends plaintext $s^d \in \mathcal{M}\setminus \{s^d_\mathcal{A}\}$ to the oracle, which will flip a coin $b\in \{0,1\}$.
    If $b=0$, $\mathcal{A}$ obtains $m_z=(m_{\text{prc}} \,\|\, m_{\text{gs}})_{s^d}$; If $b=1$, $\mathcal{A}$ obtains $u \leftarrow U_{|(m_{\text{prc}} \,\|\, m_{\text{gs}})_{s^d}|}$.
    \item \textbf{Guessing stage.} $\mathcal{A}$ output a bit $b^{\prime}$ as a ``guess'' about whether it obtains a plaintext or a random string.
\end{itemize}

Define the Chosen Plaintext Attack  (CPA) advantage of $\mathcal{A}$ against the scheme by:
\begin{equation}
        \mathbf{A d v}_{\mathsf{CS}}^{\mathrm{cpa}}(\mathcal{A}, k)\triangleq 
	\left|
                \underset{K_{\mathsf{CS}}}{\operatorname{Pr}}\left[\mathcal{A}\left(m_z\right)=1\right] -\underset{K_{\mathsf{CS}}}{\operatorname{Pr}}\left[\mathcal{A}\left(u\right)=1\right]   
        \right|.
\end{equation}

The cryptographic scheme is \textbf{indistinguishable from uniformly random bits under chosen plaintext attack (IND\$-CPA)} if $\mathbf{InSec}^{\mathrm{cpa}}_{\mathsf{CS}}(t,l,k) \triangleq \max_{\mathcal{A} \in \mathcal{A}_{(t,l)}}\{\mathbf{Adv}_{\mathsf{CS}}^{\mathrm{cpa}}(\mathcal{A},k)\}$ is negligible in $k$.
\end{definition}

\newtheorem{theorem}{Theorem}
\begin{theorem}\label{The_IND-CPA}
Let the sparsity parameter be set as \( t = \Theta(\log l(k)) \), and suppose that each encryption execution samples a fresh \( seed \) uniformly at random from \( \mathbb{F}_2^k \). Assume that the pseudorandom generator \(\mathsf{PRNG}\) satisfies standard pseudorandomness under the security parameter $k$, and that the hash function \(\mathsf{H}\) is modeled as a \textbf{random oracle}. Then the resulting ciphertext \((m_{\text{prc}} \,\|\, m_{\text{gs}})_{s^d}\) is computationally indistinguishable from a uniformly random bitstring under a chosen-plaintext attack; that is, the scheme satisfies IND\$-CPA security.

\end{theorem}
\begin{proof}
Define $H_0 \triangleq (m_{\text{prc}} \,\|\, m_{\text{gs}})_{s^d} = (G \cdot seed \oplus e) \,\|\, \left( \mathsf{PRNG}(\mathsf{H}(seed \,\|\, sk_c)) \oplus s^d \right)$.

Define $H_1$ as the variant of $H_0$ where $\mathsf{H}(seed \,\|\, sk_c)$ is replaced by a random element of \( \mathbb{F}_2^k \), i.e. $H_1 \triangleq (G \cdot seed \oplus e) \,\|\, \left( \mathsf{PRNG}(r_1) \oplus s^d \right)$.

Define $H_2$ as the variant of $H_1$ where $(G \cdot seed \oplus e)$ is replaced by a random element of $\{0,1\}^{l(k)}$, i.e. $H_2 \triangleq r_2||\left( \mathsf{PRNG}(r_1) \oplus s^d \right)$.

Define $H_3$ as the variant of $H_2$ where $\mathsf{PRNG}(r_1)$ is replaced by a random element of $\{0,1\}^{|s^d|}$, i.e. $H_3 \triangleq r_2||(r_3 \oplus s^d)$.

\begin{figure}[htbp]
	\centering
	\includegraphics[width=0.45\textwidth]{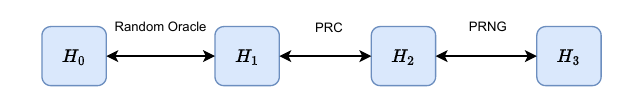}
	\caption{The hardness of distinguishing between $H_0$ and $H_3$.}
	\label{fig: fig_hardness}
\end{figure}

As shown in Fig.~\ref{fig: fig_hardness}, we claim that the advantage of distinguishing between $H_0$ and $H_1$, $H_1$ and $H_2$, $H_2$ and $H_3$, as well as $H_3$ and random bits, are all negligible in $k$.

(1) The advantage of distinguishing \( H_0 \) from \( H_1 \) is negligible if the hash function \( \mathsf{H} \) is modeled as a random oracle and the assumption that \( sk_c \) remains secret. Specifically, since the adversary does not know \( sk_c \), and \( seed \) is freshly sampled for each encryption, the value \( seed \,\|\, sk_c \) is unpredictable to the adversary. 

In the random oracle model, unless the adversary queries the oracle at exactly \( \mathsf{seed} \,\|\, sk_c \), the output \( \mathsf{H}(\mathsf{seed} \,\|\, sk_c) \) remains statistically independent of all previously seen values. Since \( sk_c \) is a \( k \)-bit secret key, each query hits the correct input with probability at most \( 2^{-k} \), making the overall success probability negligible in \( k \). Therefore, replacing \( \mathsf{H}(\mathsf{seed} \,\|\, sk_c) \) with a truly random string \( r_1 \leftarrow \mathbb{F}_2^k \) results in a distribution that is computationally indistinguishable from the original.

(2) Distinguishing \( H_1 \) from \( H_2 \) would contradict the pseudorandomness of the PRC construction. Given that the sparsity parameter is set as \( t = \Theta(\log l(k)) \), and each encryption execution samples a fresh \( seed \) uniformly at random from \( \mathbb{F}_2^k \), the term \( G \cdot seed \oplus e \) constitutes an LDPC codes-based PRC codeword. As shown in Sec.~\ref{sec:ldpc-prc}, the pseudorandomness of this construction is guaranteed either by the subexponential LPN assumption, or by the standard LPN assumption in combination with the planted XOR assumption.

Moreover, since each PRC codeword is generated using an independently sampled random $seed$, ciphertexts corresponding to different queries are statistically independent. This ensures that the construction satisfies chosen-plaintext attack (CPA) security under the assumed hardness of the underlying PRC.

(3) Distinguishing \( H_2 \) from \( H_3 \) would contradict the security of the $\mathsf{PRNG}$. By assumption, the $\mathsf{PRNG}$ satisfies standard pseudorandomness under the security parameter \( k \); that is, for a uniformly sampled seed \( r_1 \in \{0,1\}^k \), the output \( \mathsf{PRNG}(r_1) \) is computationally indistinguishable from a uniformly random bitstring of the same length. Therefore, if an adversary could distinguish \( H_2 \) from \( H_3 \), it would imply a distinguisher against the $\mathsf{PRNG}$. 

In addition, since each encryption instance uses a freshly sampled $seed$, the input to the $\mathsf{PRNG}$ is independent across queries. Together with the pseudorandomness of $\mathsf{PRNG}$, this ensures that ciphertexts are CPA-secure even under multiple chosen plaintexts.

(4) Distinguishing \( H_3 \) from a uniformly random bitstring is information-theoretically impossible, as the distribution of \( H_3 \) is identical to the uniform distribution over \( \{0,1\}^{l(k) + |s^d|} \).

By a sequence of hybrid arguments, we have shown that the ciphertext generated by the scheme is computationally indistinguishable from a uniformly random bitstring under a chosen-plaintext attack. Therefore, the scheme satisfies IND\$-CPA security under the random oracle model, assuming the pseudorandomness of the $\mathsf{PRNG}$, the pseudorandomness of the PRC, and the secrecy of the private key \( sk_c \).
\end{proof}

\newtheorem{remark}{Remark}

\begin{remark}
As demonstrated above, our hybrid watermarking scheme inherits the pseudorandomness guarantees established in prior theoretical work. This represents a conceptual advancement over PRCW~\cite{prc}, whose watermark ciphertext takes the form \( G \cdot (testbits \,\|\, m \,\|\, r) \oplus e \). Since $testbits$ and \( m \) are not randomly sampled during use, PRCW cannot theoretically reduce its performance-lossless property to the hardness assumptions underlying LDPC codes-based PRC. In practical scenarios where \( m \) remains fixed, the pseudorandomness of PRCW relies solely on \( r \), making it potentially vulnerable under large sample sizes. This limitation in pseudorandomness may be empirically observed, as discussed in Sec.~\ref{sec:Latent Distribution}.

This proof above is conducted in the context of \textbf{Operator Verification}, where our watermarking scheme is shown to be provably performance-lossless. When extended to the \textbf{Third-party Verification} setting, we no longer consider the undetectability of the watermark compared to the unwatermarked image, since public verifiability inherently conflicts with undetectability. This is because undetectability requires that no polynomial-time algorithm can distinguish between the two, whereas in the \textbf{Third-party Verification} setting, the watermark extraction algorithm is publicly available and can naturally be used for detection. 
\end{remark}

\begin{figure*}[htbp]
    \centering
    \subfloat[Original]{\includegraphics[width=.19\linewidth]{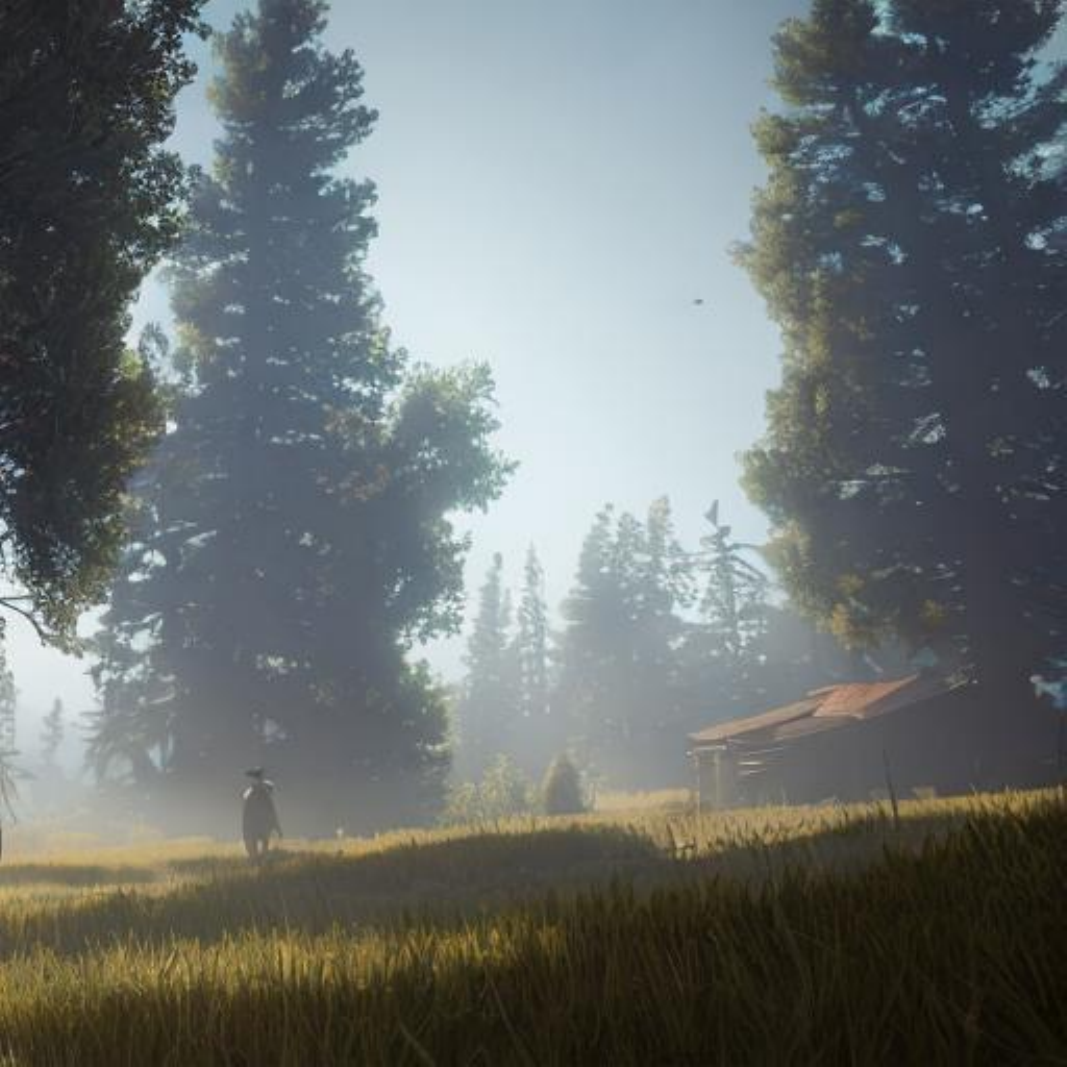}}\hspace{0.005\linewidth}
    \subfloat[DwtDct]{\includegraphics[width=.19\linewidth]{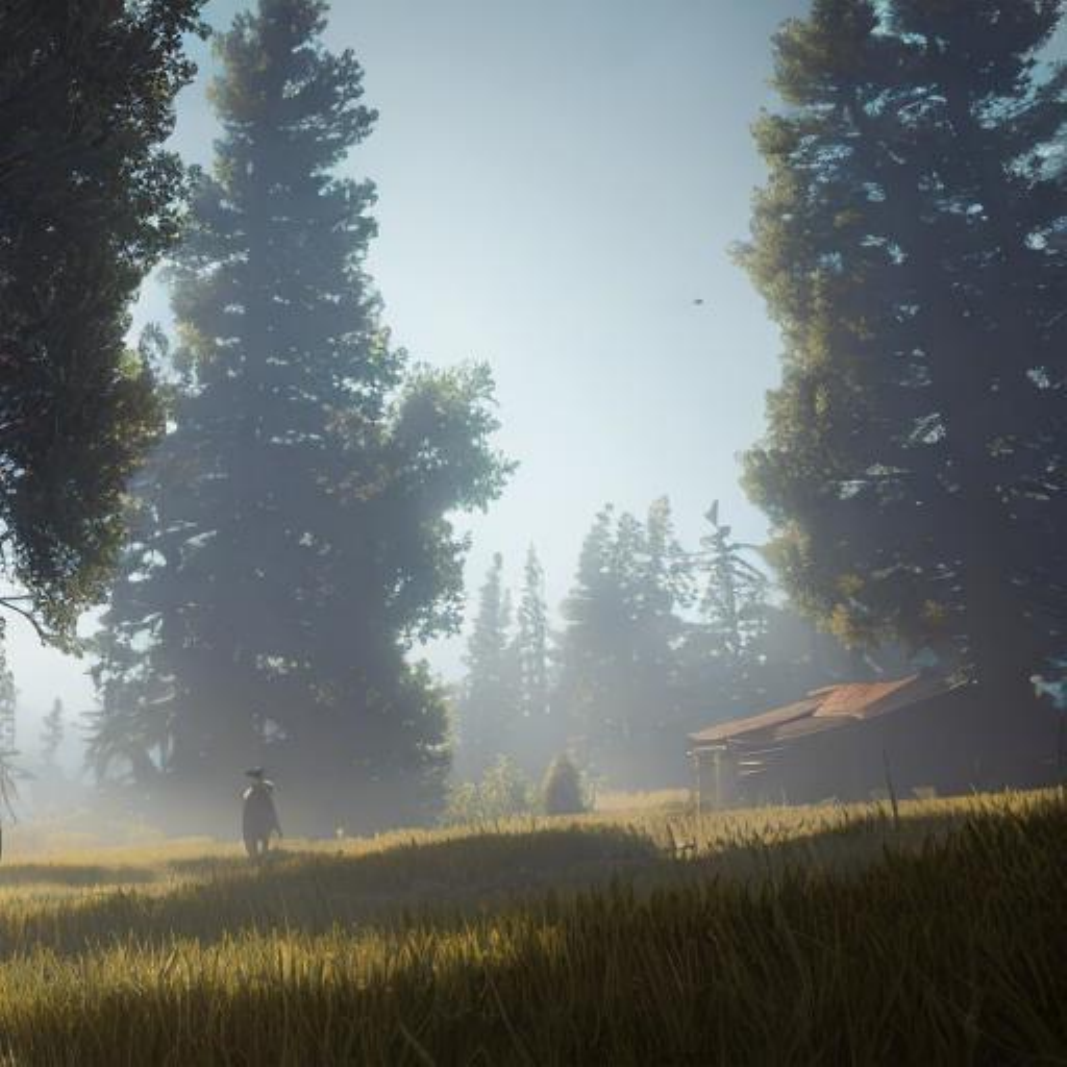}}\hspace{0.005\linewidth}
    \subfloat[DwtDctSvd]{\includegraphics[width=.19\linewidth]{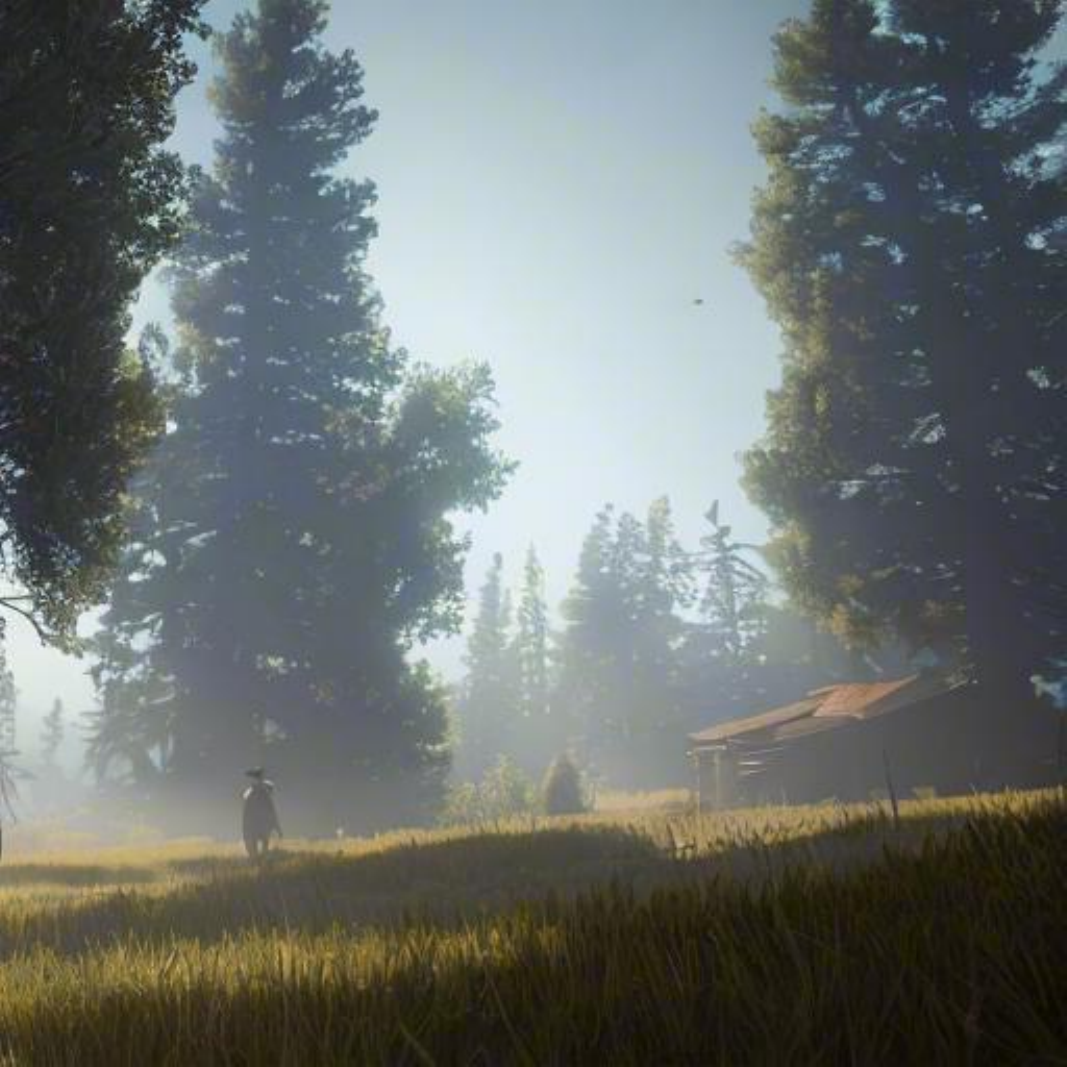}}\hspace{0.005\linewidth}
    \subfloat[RivaGAN]{\includegraphics[width=.19\linewidth]{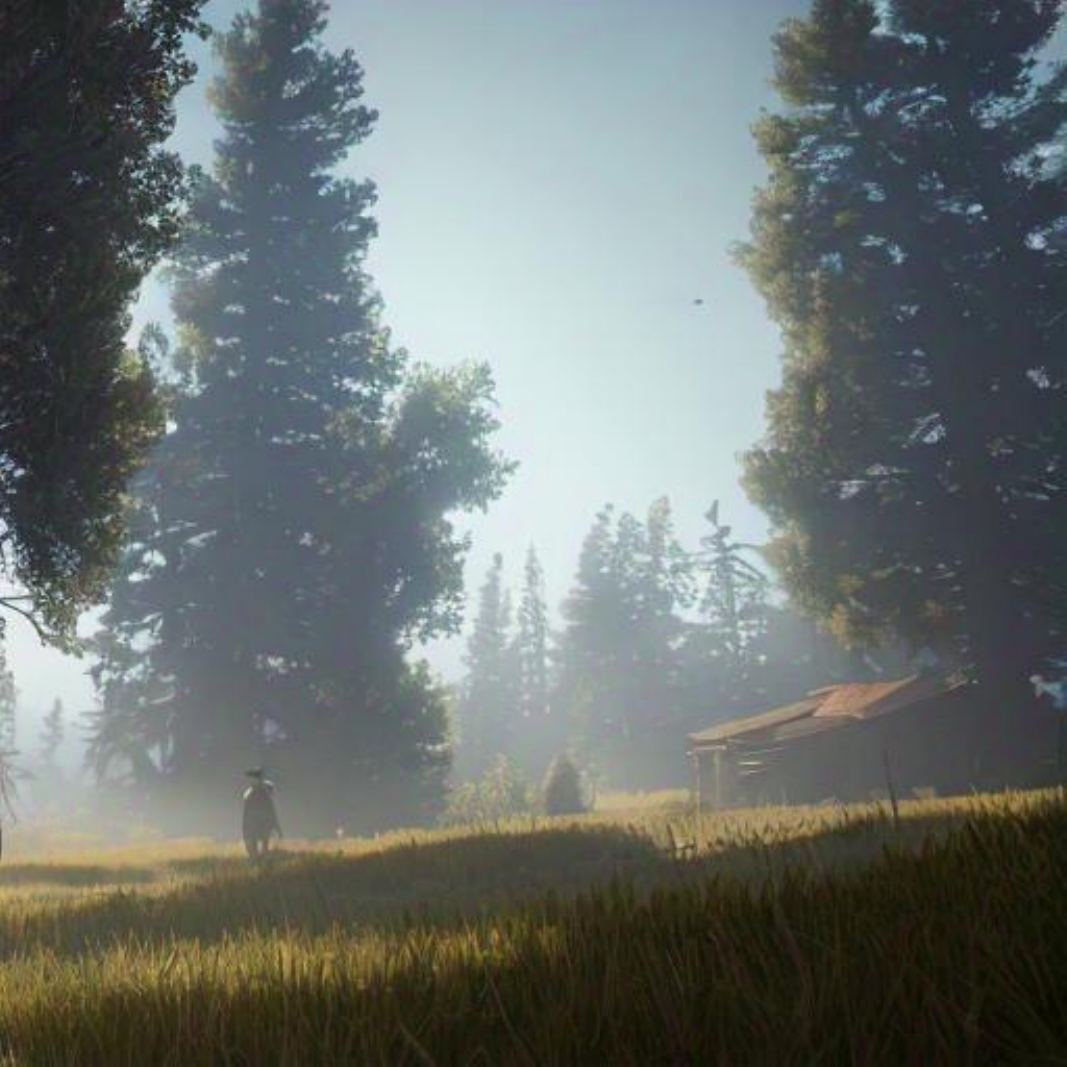}}\hspace{0.005\linewidth}
    \subfloat[Stable Signature]{\includegraphics[width=.19\linewidth]{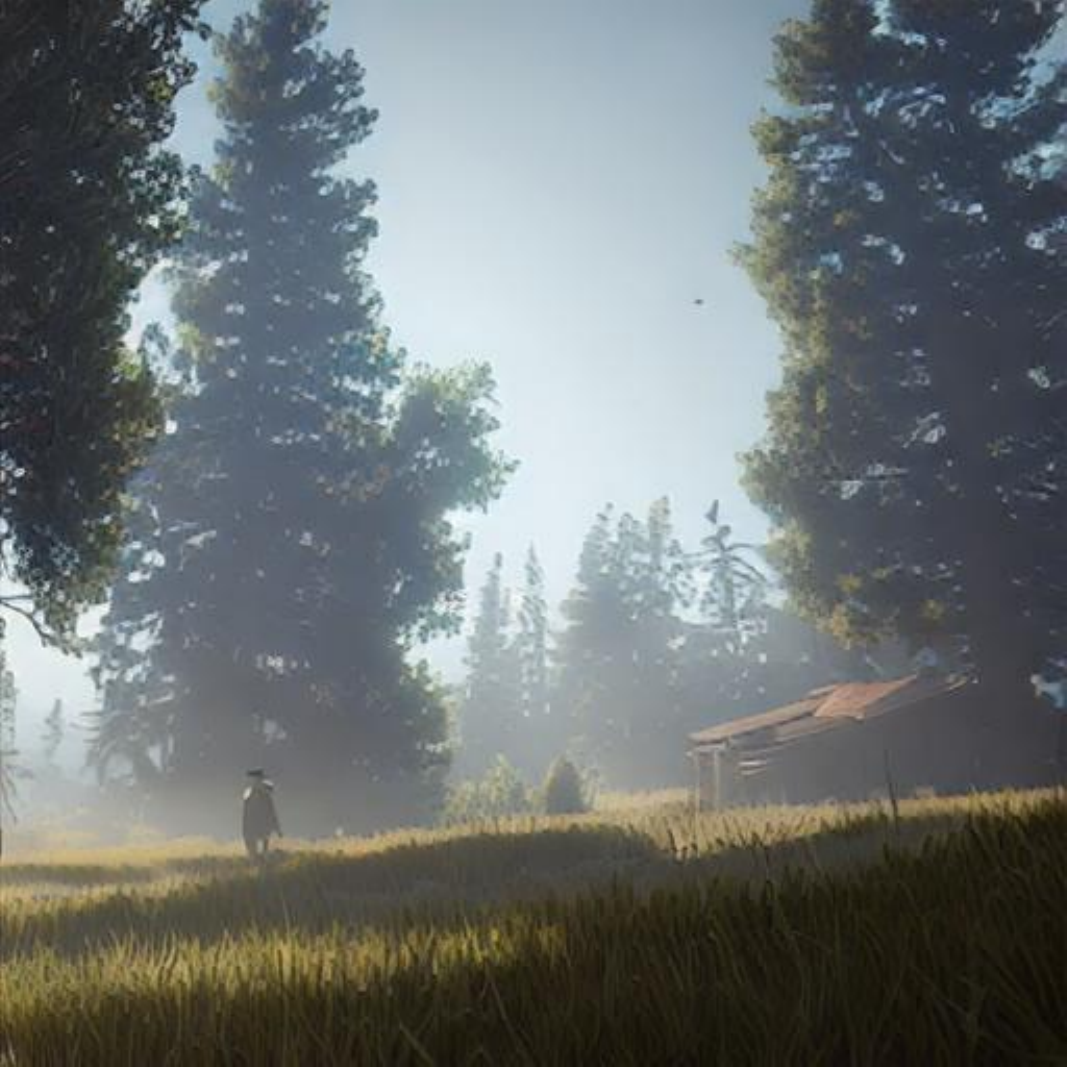}}

    \hspace{0.195\linewidth} 
    \subfloat[TRW]{\includegraphics[width=.19\linewidth]{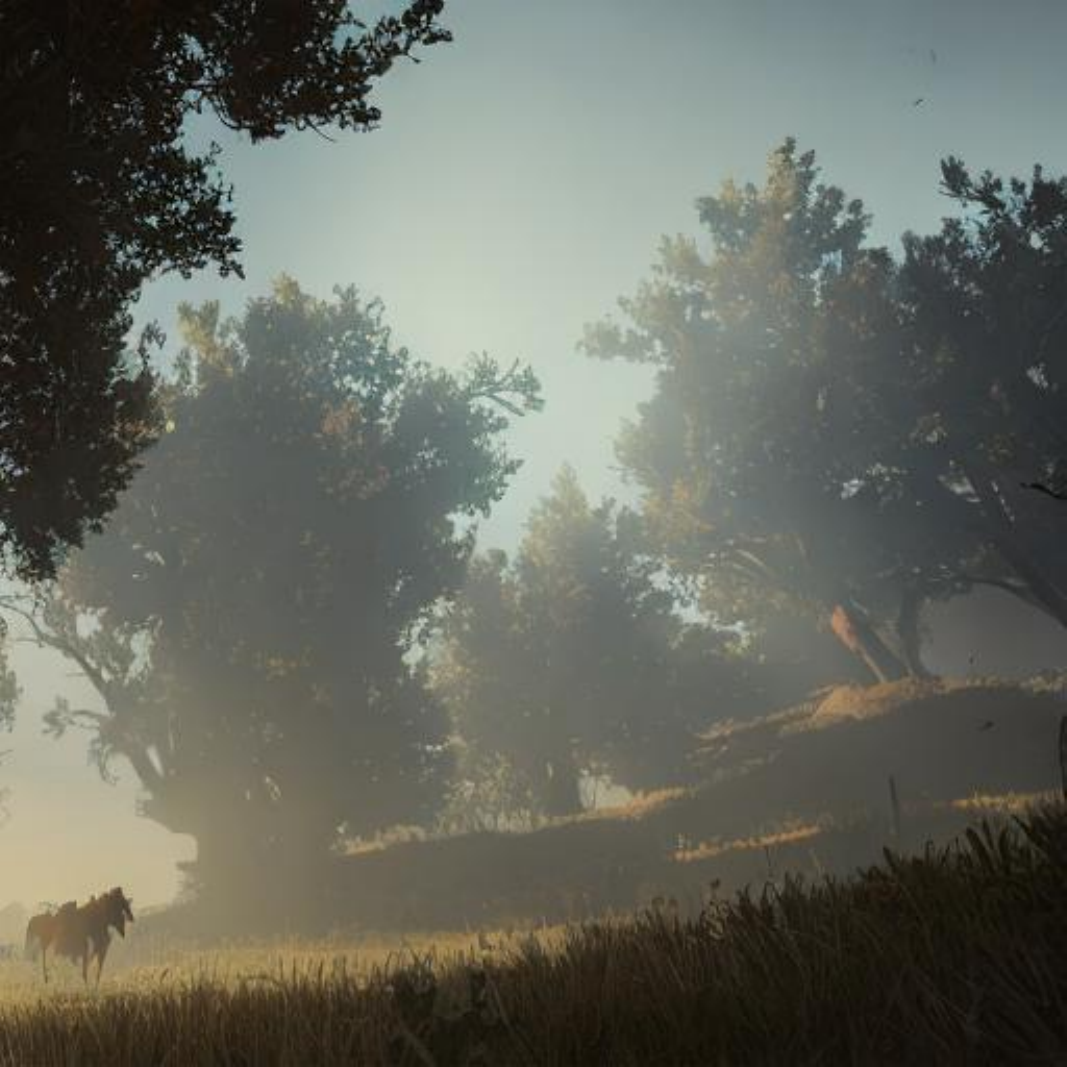}}\hspace{0.005\linewidth}
    \subfloat[Gaussian Shading]{\includegraphics[width=.19\linewidth]{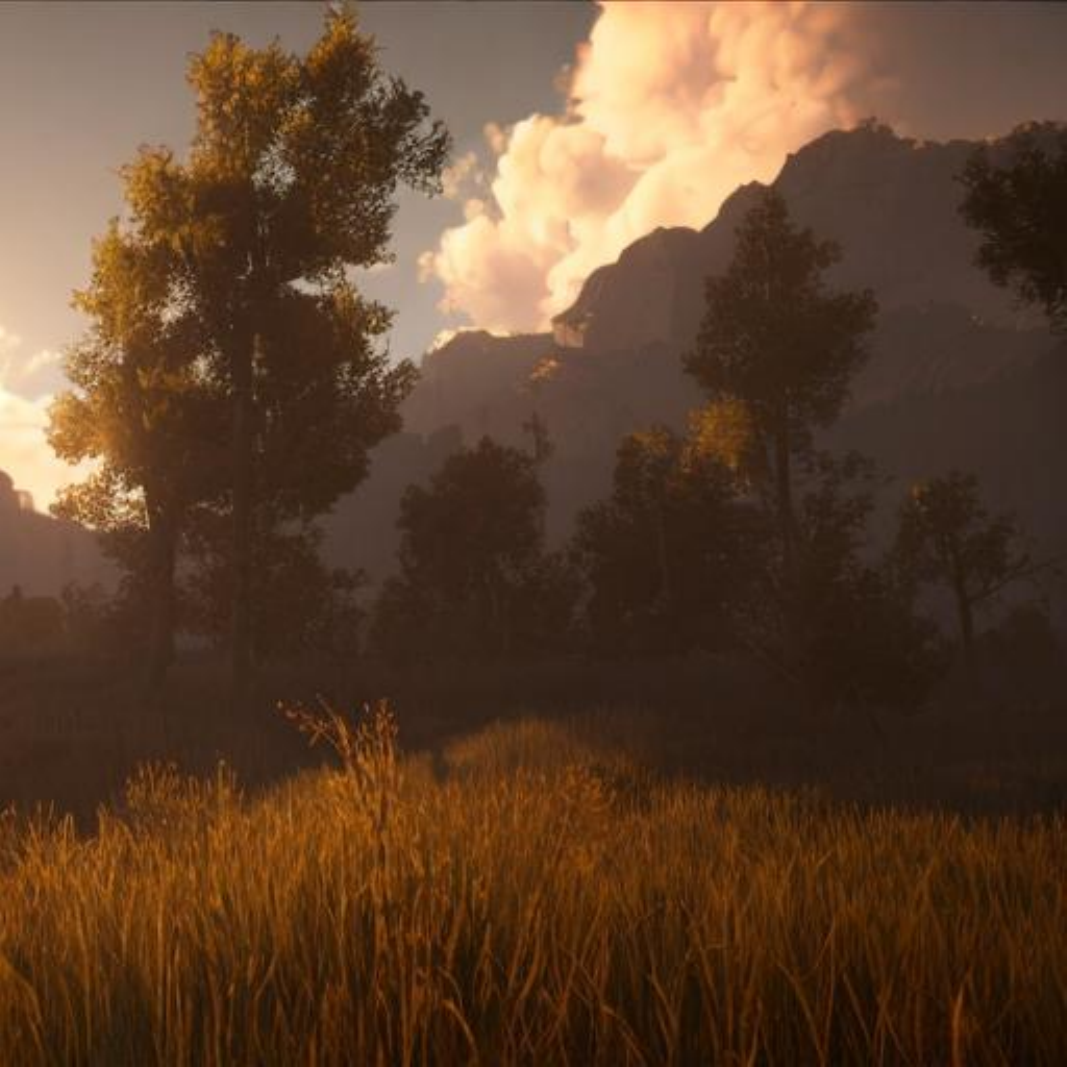}}\hspace{0.005\linewidth}
    \subfloat[PRCW]{\includegraphics[width=.19\linewidth]{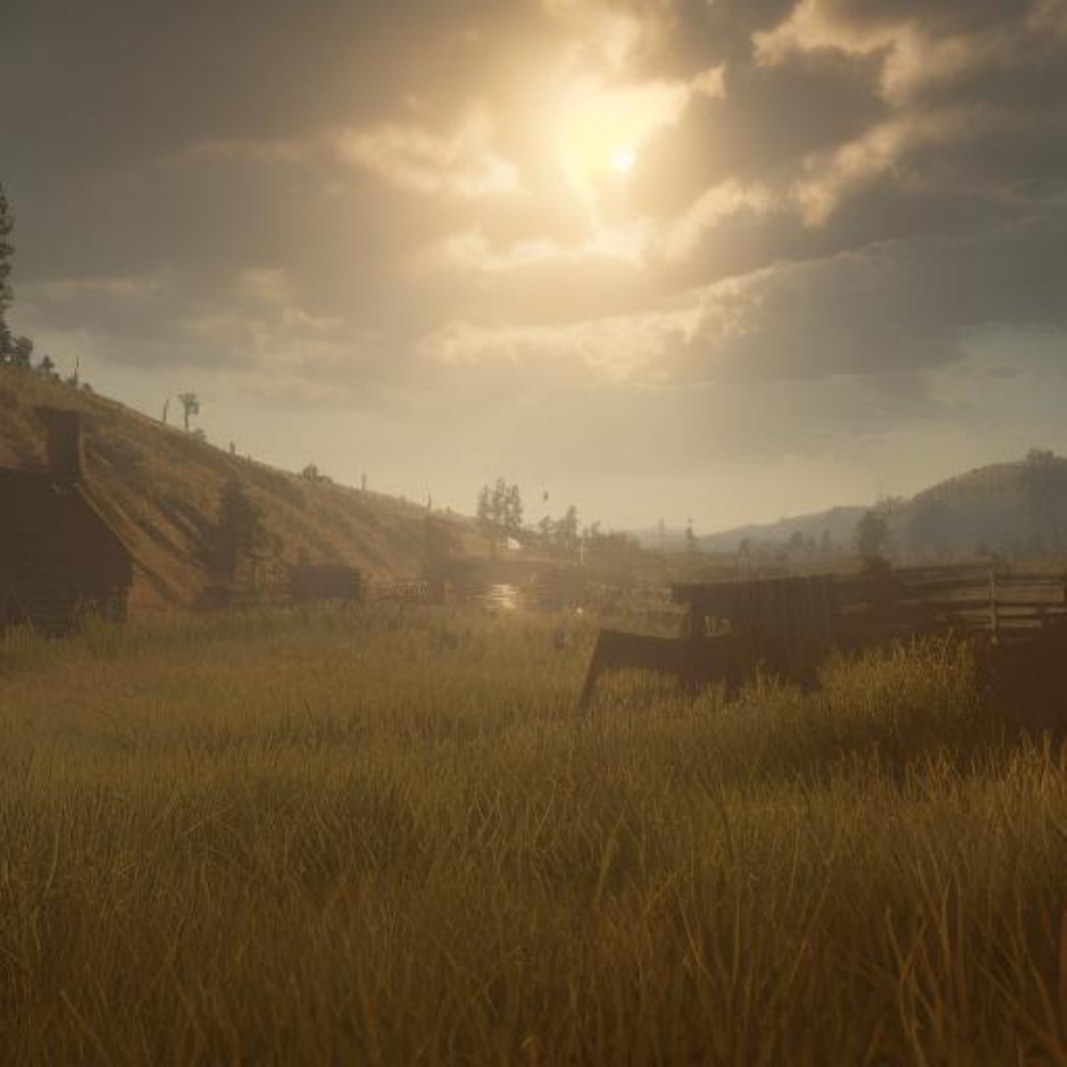}}\hspace{0.005\linewidth}
    \subfloat[Gaussian Shading++]{\includegraphics[width=.19\linewidth]{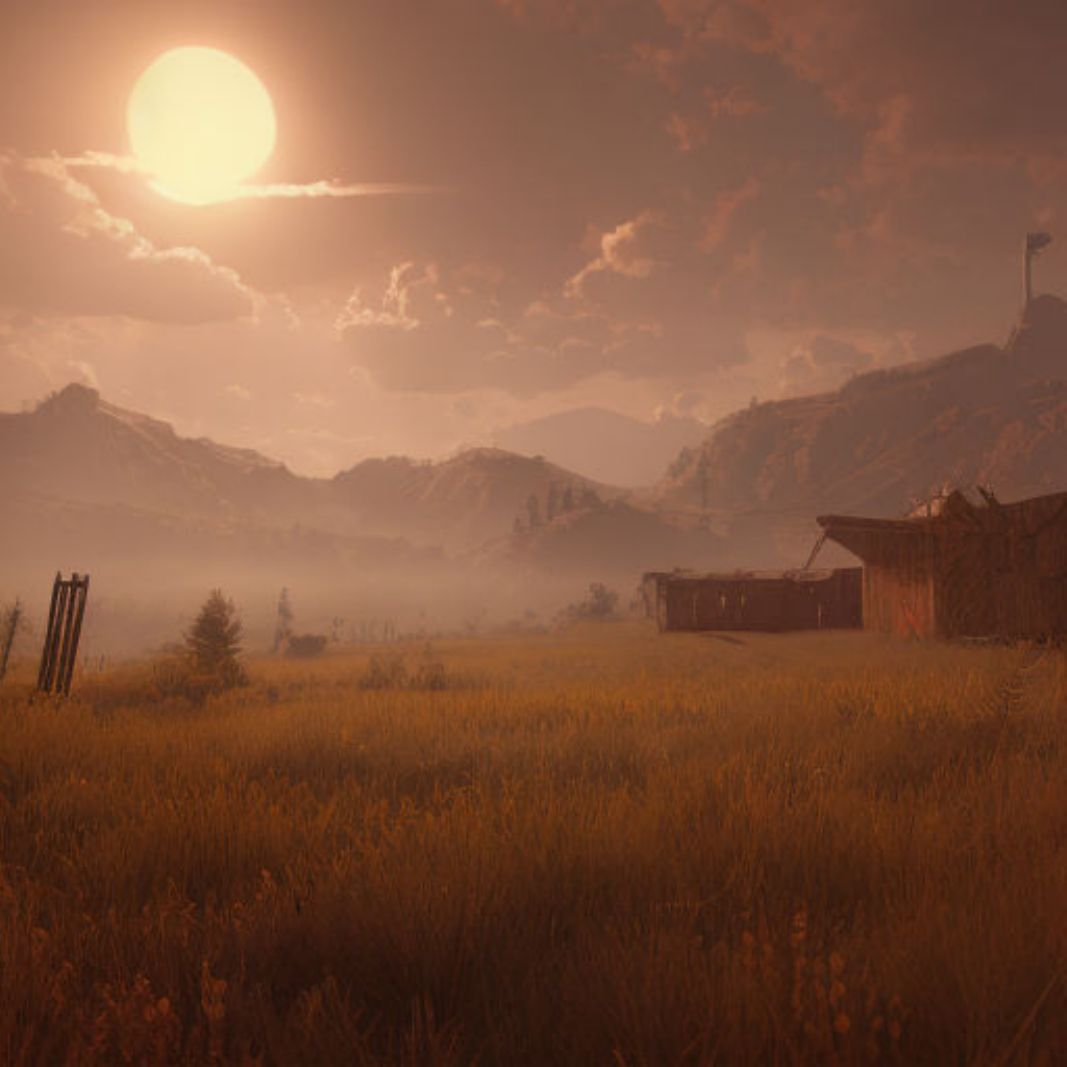}}

    \caption{Watermarked images generated using different watermarking methods with the same prompt: ``\textit{Red dead redemption 2, cinematic view, epic sky, detailed, concept art, low angle, high detail, warm lighting, volumetric, godrays, vivid, beautiful, trending on artstation, by jordan grimmer, huge scene, grass, art greg rutkowski.}". Among them, post-processing-based methods (b) (c) (d) and the fine-tuning-based method (e) only add image residuals compared to the original image (a).}
    \label{fig:wm}
\end{figure*}
\section{Experiments}
This section presents the experimental analysis. We begin by evaluating the performance of Gaussian Shading++ in both the Operator Verification and Third-party Verification scenarios. Furthermore, we conduct comprehensive comparisons with several state-of-the-art methods. Lastly, to validate the effectiveness of each component within Gaussian Shading++, we perform thorough ablation studies.

\subsection{Implementation Details}
\subsubsection{Diffusion Models} In this paper, we focus on text-to-image latent diffusion models, hence we select the Stable Diffusion (SD)~\cite{rombach2022high}  provided by huggingface. We evaluate Gaussian Shading++ as well as baseline methods, using SD V2.1. The size of the generated images is $512\times 512$, and the latent space dimension is $4 \times 64 \times 64$. During inference, we adopt prompts from the Stable-Diffusion-Prompt (SDP) dataset\footnote{\href{https://huggingface.co/datasets/Gustavosta/Stable-Diffusion-Prompts}{Stable-Diffusion-Prompts}}, using a guidance scale of $7.5$ (the default setting in SD). Image generation is performed with $50$ sampling steps using DPMSolver~\cite{lu2022dpm}.
Considering that users typically share generated images without preserving the original prompts, we perform $50$ steps of inversion using Exact Inversion~\cite{Hong_2024_CVPR} with an empty prompt and a guidance scale of $3$ (the default setting in Exact Inversion~\cite{Hong_2024_CVPR}).

\subsubsection{Watermark Methods} In the main experiments, The PRC Channel of Gaussian Shading++ encodes a $32$-bit 
$seed$ using the PRC. For the GS Channel, we set the parameters as $f_{ch}=2, f_{hw}=4, v=1$, resulting in an actual capacity of $256$ bits. In the Third-party Verification scenario, we employ the public-key signature scheme ECDSA~\cite{ECDSA},  the $256$-bit watermark is further split into $32$ bits of user information and $224$ bits of signature.

For comparison, we select three representative categories of watermark methods as baselines:
\begin{itemize}
    \item Post-processing-based: We adopt three officially used by SD, which embed specific patterns in either the frequency or spatial domain: DwtDct~\cite{cox2007digital}, DwtDctSvd~\cite{cox2007digital}, and RivaGAN~\cite{zhang2019robust}. The capacities of DwtDct~\cite{cox2007digital} and DwtDctSvd~\cite{cox2007digital} are set to $256$ bits. Since RivaGAN~\cite{zhang2019robust} supports a maximum capacity of $32$ bits, we retain this setting.

    \item Fine-tuning-based: We adopt Stable Signature~\cite{fernandez2023stable}, which injects watermark information by fine-tuning the VAE decoder of Stable Diffusion, ensuring the watermark is embedded in the generated images. Due to convergence issues caused by large watermark capacities, we instead use the official open-source model of Stable Signature with a capacity of $48$ bits\footnote{\href{https://github.com/facebookresearch/stable_signature}{The GitHub Repository for Stable Signature}}. During fine-tuning, we use $400$ images from the ImageNet2014~\cite{deng2009imagenet} validation set, with a batch size of $4$ and $100$ training steps.

    \item Latent-representation-based: We adopt TRW~\cite{wen2023tree} , Gaussian Shading~\cite{gs}, and PRCW~\cite{prc}. TRW~\cite{wen2023tree} embeds the watermark by modifying latent representations to align with specific patterns. As it is a single-bit watermark, we evaluate it only in the detection task. Since its Rand mode better aligns with the notion of performance-lossless watermark, we adopt this setting. Gaussian Shading~\cite{gs} uses a stream cipher to pseudorandomize the watermark. However, considering the challenges of key management in practical deployment, the stream cipher is fixed during our experiments. PRCW~\cite{prc} encodes the watermark using pseudorandom error-correcting codes. For a fair comparison, the watermark capacities of both Gaussian Shading~\cite{gs} and PRCW~\cite{prc} are set to $256$ bits.
\end{itemize}

Examples of images from all of the above watermarking methods are shown in Fig.~\ref{fig:wm}.

\subsubsection{Robustness Evaluation}  We evaluate the robustness of the methods from two perspectives: traditional noise distortions and removal attack~\cite{zhao2024invisible}. For traditional distortions, we select six representative types of noise: (a) JPEG Compression, $QF=70$ (JPEG). (b) Brightness, $factor=1$. (c) Gaussian Blur, $radius= 3$ (GauBlur). (d) Gaussian Noise, $\mu = 0$, $\sigma = 0.01$ (GauNoise). (e) Median Filtering, $kernel\_size=7$ (MedFilter). (f) 50\% Resize and restore (Resize).  
For removal attack, we adopt Variational AutoEncoder (VAE)~\cite{balle2018variational,cheng2020learned,esser2021taming,rombach2022high}, and Stable Diffusion (SD)~\cite{rombach2022high} as erasure networks.

\begin{figure*}[t]
    \centering
    \subfloat[Detection results.]{\label{Fig:tpr}\includegraphics[width=0.4\linewidth]{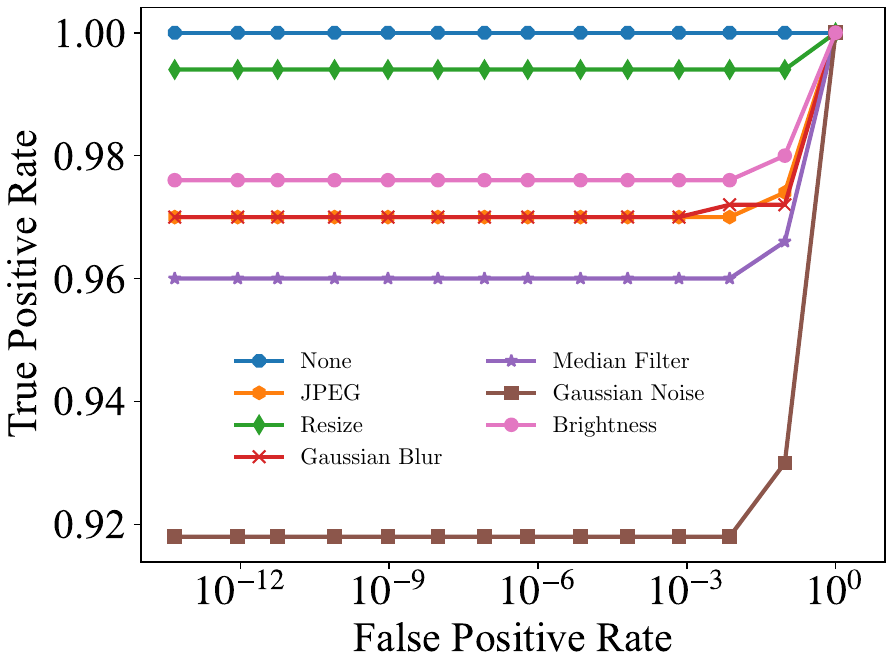}}\hspace{0.1\linewidth}
     \subfloat[Traceability results.]{\label{Fig:t_acc}\includegraphics[width=0.4\linewidth]{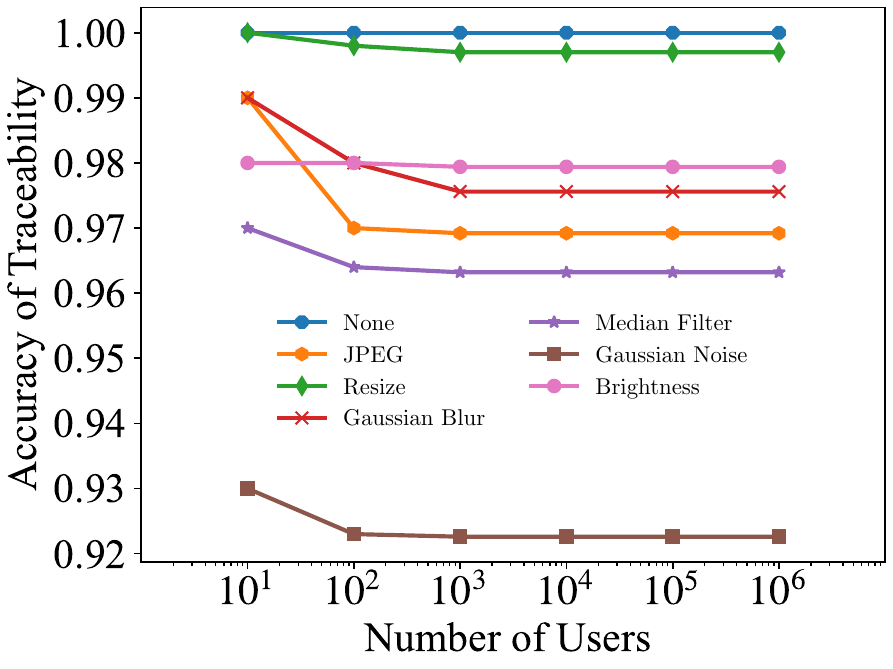}}
    \caption{Performance of Gaussian Shading++ in Operator Verification scenario. The results are presented separately for detection and traceability tasks.}
    \label{fig:performance}
\end{figure*}

\subsubsection{Evaluation Metrics} 
In the Operator Verification scenario, we focus on two primary tasks: detection and traceability. For detection, we compute the bit accuracy threshold corresponding to a fixed false positive rate (FPR) based on Eq.~\ref{smeq:detection}, and then report the true positive rate (TPR) on watermarked images. For traceability, we directly evaluate the bit accuracy of watermark extraction.
In the Third-party Verification scenario, we evaluate the accuracy of traceability, measuring the success rate of accurately tracing the watermark back to the target user.

Apart from the above metrics, we assess the impact of each method on model performance from two perspectives: visual quality and distribution of latent representations.
For visual quality, we employ Fréchet Inception Distance (FID)~\cite{heusel2017gans} and CLIP-Score~\cite{radford2021learning} to measure the realism and text-image alignment of watermarked images, respectively. Specifically, we compute FID and CLIP-Score over $10$ batches of watermarked images and conduct a $t$-test comparing the mean values against those of non-watermarked images. A smaller $t$-value indicates less performance degradation caused by watermarking.
For the distribution of latent representations, which mainly targets latent-representation-based methods, we generate $80,000$ latent representations and model them as samples from a standard Gaussian distribution. We then perform hypothesis testing, including the K-S test, Shapiro-Wilk test,  and other statistical methods, to compute the $p$-value. A higher $p$-value indicates a closer match to the standard Gaussian distribution.

All experiments are conducted using the PyTorch 2.1.0
framework, running on a single RTX A6000 GPU.

\subsection{Performance of Gaussian Shading++}

\subsubsection{Operator Verification}

To enable detection, we consider Gaussian Shading++ as a single-bit watermark, with a fixed watermark $s$. We approximate the FPR to be controlled at $10^0,10^{-1},\dots,10^{-13}$, calculate the corresponding threshold $\tau$ 
 , and test the TPR on $500$ watermarked images. As shown in Fig.~\ref{Fig:tpr}, when the FPR is controlled at $10^{-13}$, the TPR remains at least $0.97$  for five out of the seven cases. Although the TPR for Gaussian Noise is only $0.918$, it is still a promising result.

For traceability, Gaussian Shading++ serves as a multi-bit watermark. Assuming Alice provides services to $N$ users, Alice needs to allocate one watermark for each user. In our experiments, we assume that $N' = 1,000$ users generate images, with each user generating $5$ images, resulting in a dataset of $5,000$ watermarked images.

During testing, we calculate the threshold $\tau$  
to control the FPR at $10^{-6}$. Note that when computing traceability accuracy, we need to consider two types of errors: false positives, where watermarked images are not detected, and traceability errors, where watermarked images are detected but attributed to the wrong user. Therefore, we first  
determine whether the image contains a watermark. If it does, we calculate the number of matching bits $Acc$ with all $N$ users on the platform. The user with the highest $Acc$ is considered the one who generated the image. Finally, we verify whether the correct user has been traced. When $N > N'$, it can be assumed that some users have been assigned a watermark but have not generated any images.

As shown in~Fig.~\ref{Fig:t_acc}, 
when $N = 10^6$, Gaussian Shading++ achieves a traceability accuracy of over $96\%$ in  six cases. Although the traceability accuracy for Gaussian Noise is only $92.26\%$, if a user generates two images, the probability of successfully tracing him is still no less than $99\%$.

\subsubsection{Third-party Verification}

In this scenario, any third party can perform traceability on the target user using the watermark database publicly provided by the operator. Since the public-key signature is unforgeable, the probability that a randomly generated signature passes verification is negligible. Therefore, the traceability accuracy is effectively equivalent to the signature verification success rate. We generate $500$ images and evaluate the traceability accuracy of Gaussian Shading++ under various traditional noise distortions, both before and after incorporating ECDSA. As shown in Tab.~\ref{tab:ecdsa}, after introducing the public-key signature, Gaussian Shading++ experiences a significant performance drop when facing filtering-based distortions such as Gaussian Blur and Median Filter. However, it still maintains robust performance under other conditions, with traceability accuracy remaining above $70\%$.

\begin{table}[htbp]
	\centering
 	\caption{Traceability accuracy of Gaussian Shading++. ``Ours" denotes the naive version, while ``Ours + ECDSA" represents the version with the public-key signature.} \label{tab:ecdsa}
       \begin{tabular}{@{}ccc@{}}
    \toprule
    \multirow{2}{*}{\textbf{Noise}} &
		\multicolumn{2}{c}{\textbf{Methods}} \\
		\cmidrule{2-3}
		&Ours  & Ours + ECDSA\\
\midrule
None&1.000 &0.994   \\
JPEG & 0.972&   0.722\\
Brightness & 0.980&  0.896 \\
GauBlur & 0.976&   0.158\\
GauNoise   & 0.922&   0.714\\
MedFilter   & 0.964& 0.358\\
Resize  & 0.996&  0.870      \\
\bottomrule
\end{tabular}
\end{table}

After the operator publicly releases the model inversion capability, malicious users can reuse latent representations and generate forged images with illicit prompts on a proxy model to falsely frame target users, thereby launching a reprompt attack (``Attk")~\cite{müller2024blackbox}. Moreover, since the sign of the latent representations in Gaussian Shading++ encodes the watermark, attackers can enhance the attack (``Attk+") by resampling within specific intervals corresponding to the watermark bits.

In our experiments, we consider five potential proxy models that attackers might use: SD V1.4, SD V1.5, SD V2.0, SD V2.1~\cite{rombach2022high}, and SD-XL~\cite{podell2023sdxl}. The first $500$ prompts from the I2P~\cite{schramowski2023safe} dataset are used as illicit prompts. For the ``Attk" scenario, we use SD V2.1 to generate $500$ forged images. For the enhanced ``Attk+" scenario, we perform three resampling attempts for each illicit prompt, resulting in a total of $1,500$ forged images. We compared the attack success rate based on the use of ECDSA, and the results are presented in Tab.~\ref{tab:reprompt}. It is evident that introducing the public-key signature significantly reduces the risk of Gaussian Shading++ being forged, particularly when the proxy models are SDV1.4 and SDV1.5. Since SD-XL has a model architecture and parameters that are substantially different from SD V2.1, forgery is difficult to achieve. However, due to the similarity in parameters between SD V2.0 and SD V2.1, these two proxy models can easily perform forgery, which further emphasizes the importance of the operator safeguarding the model parameters.

Overall, although introducing ECDSA may compromise some robustness, it significantly reduces the risk of successful forgery when attackers cannot access precise model parameters. This trade-off is worthwhile when extending the functionality to the Third-party Verification scenario.

\begin{table}[htbp]
\centering
  \caption{Attack success rate of reprompt attacks (``Attk" / ``Attk+") on Gaussian Shading++. ``Ours" denotes the naive version, while ``Ours + ECDSA" represents the version with the public-key signature. }
  \label{tab:reprompt}
       \begin{tabular}{@{}ccc@{}}
    \toprule
    \multirow{2}{*}{\textbf{Proxy Model}} &
		\multicolumn{2}{c}{\textbf{Methods}} \\
		\cmidrule{2-3}
		&Ours  & Ours + ECDSA\\
\midrule
SD V1.4 & 0.894 / 0.937 & \textbf{0.483} / \textbf{0.523}     \\
SD V1.5 & 0.886 / 0.937 & \textbf{0.362} / \textbf{0.433}     \\
SD V2.0 & 0.941 / 0.985 & \textbf{0.832} / \textbf{0.889}   \\
SD V2.1 & 0.944 / 0.985 & \textbf{0.839} / \textbf{0.889}    \\
SD-XL   & \textbf{0.000} / \textbf{0.000}        & \textbf{0.000} / \textbf{0.000}             \\
\bottomrule
\end{tabular}
\end{table}

\begin{table*}[t]
  \centering
    \caption{Comparison of robustness under traditional noise distortions. We control the FPR at $10^{-6}$, and evaluate the TPR / bit accuracy for SD V2.1. \textbf{Bold} represents the best, \underline{underline} represents the second best. }
  \label{tab:com_noise}

  \begin{tabular}{@{}cccccccc@{}}
    \toprule
    \multirow{2}{*}{\textbf{Methods}} &
		\multicolumn{7}{c}{\textbf{Noise}} \\
		\cmidrule{2-8}
		&\textbf{None}& \textbf{JPEG} &\textbf{Brightness}&\textbf{GauBlur}&\textbf{GauNoise} &\textbf{MedFilter}&\textbf{Resize} \\
    \midrule
    DwtDct~\cite{cox2007digital} &0.832 / 0.790&0.000 / 0.509&0.292 / 0.579&0.002 / 0.509&0.814 / 0.770&0.000 / 0.521&0.218 / 0.594\\
    DwtDctSvd~\cite{cox2007digital} &\textbf{1.000} / \underline{0.999}&\underline{0.998} / 0.870&0.186 / 0.503&0.904 / 0.771&\underline{0.998} / \underline{0.979}&\underline{0.998} / 0.937&\underline{0.996} / 0.979\\
    
    RivaGAN~\cite{zhang2019robust} &\underline{0.970} / 0.993&0.844 / 0.964&0.764 / 0.936&0.718 / 0.942&0.738 / 0.933&0.914 / 0.970&0.940 / 0.986\\
    
    Stable Signature~\cite{fernandez2023stable} &\textbf{1.000} / 0.998&0.856 / 0.889&0.886 / 0.937&0.000 / 0.413&0.948 / 0.973&0.000 / 0.647&0.332 / 0.806\\
    
    TRW~\cite{wen2023tree} &\textbf{1.000} / -&\textbf{1.000} / -&0.906 / -&\textbf{1.000} / -&0.730 / -&\textbf{1.000} / -&\textbf{1.000} / -\\
    
    Gaussian Shading~\cite{gs} &\textbf{1.000} / \textbf{1.000}&\textbf{1.000} / \textbf{0.999}&\textbf{0.998} / \textbf{0.999}&\textbf{1.000} / \textbf{0.998}&\textbf{1.000} / \textbf{0.998}&\textbf{1.000} / \textbf{0.998}&\textbf{1.000} / \textbf{1.000}\\
    
    PRCW~\cite{prc} &\textbf{1.000} / \textbf{1.000}&0.950 / 0.972&0.964 / 0.984&0.846 / 0.933&0.878 / 0.943&0.886 / 0.943&0.988 / 0.992\\
    
    \textbf{Gaussian Shading++} &\textbf{1.000} / \textbf{1.000}&0.974 / \underline{0.984}&\underline{0.974} / \underline{0.986}&\underline{0.974} / \underline{0.972}&0.918 / 0.956&0.964 / \underline{0.974}&\underline{0.996} / \underline{0.996}\\
    \bottomrule
    
  \end{tabular}
\end{table*}

\begin{figure*}
    \centering
    {\label{Fig:tpr_legend}\includegraphics[width=.16\linewidth]{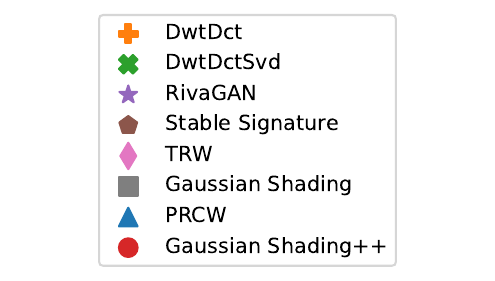}}
    \subfloat[VAE B]{\label{Fig:bmshj_tpr}\includegraphics[width=.15\linewidth]{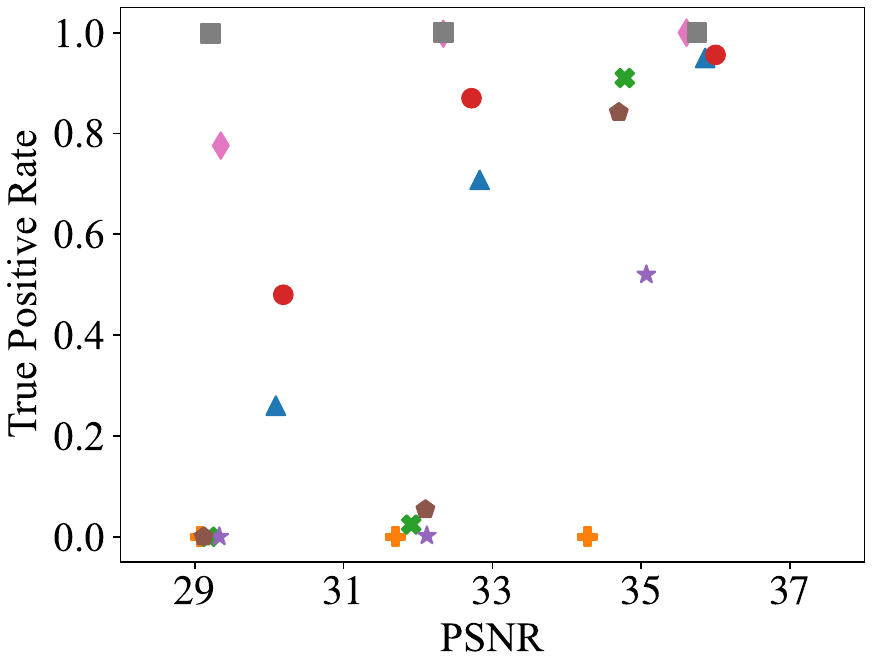}}\hspace{0.005\linewidth}
    \subfloat[VAE C]{\label{Fig:cheng_tpr}\includegraphics[width=.15\linewidth]{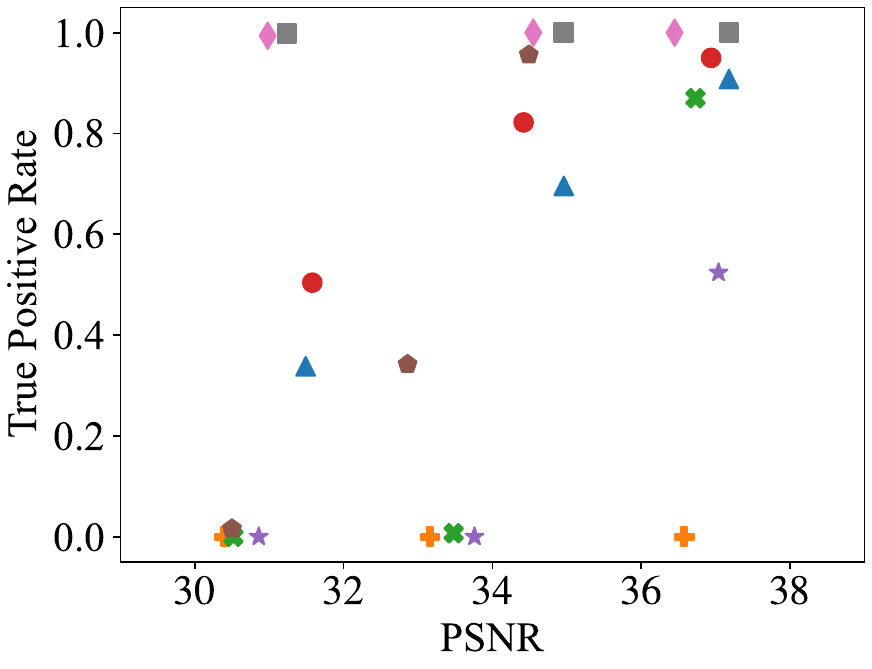}}\hspace{0.005\linewidth}
    \subfloat[VQ VAE]{\label{Fig:vq_tpr}\includegraphics[width=.15\linewidth]{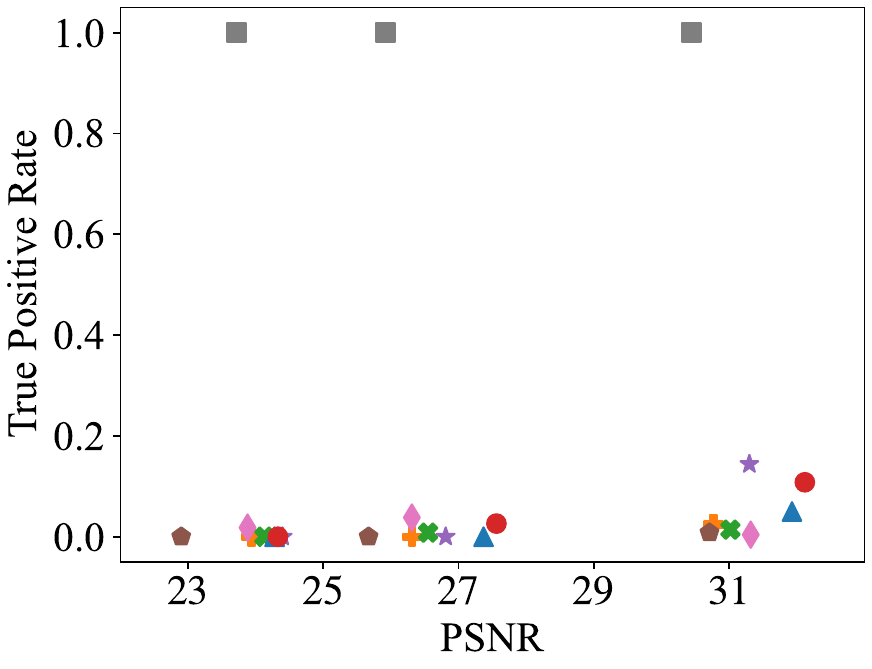}}\hspace{0.005\linewidth}
 \subfloat[KL VAE]{\label{Fig:kl_tpr}\includegraphics[width=.15\linewidth]{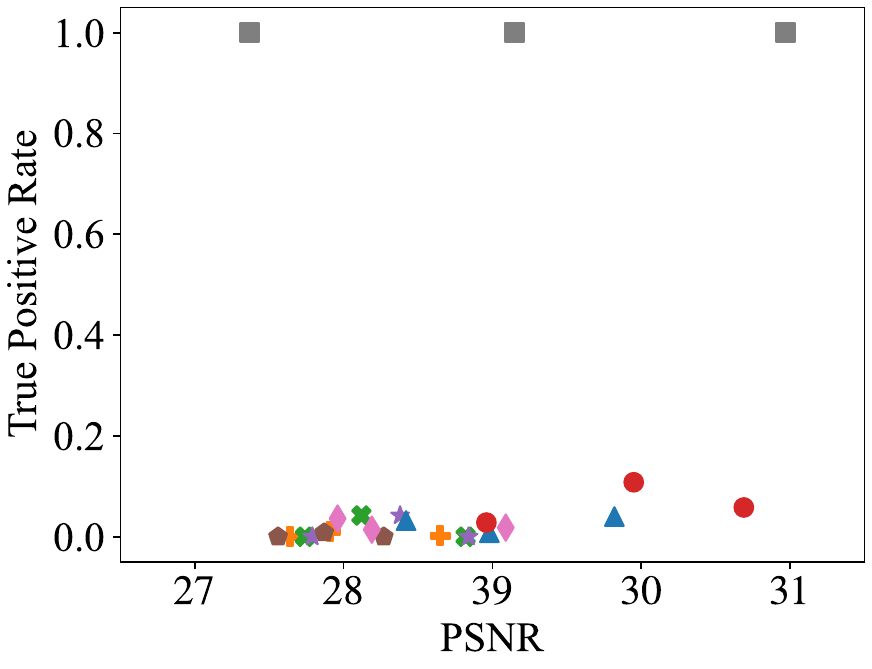}}\hspace{0.005\linewidth}
 \subfloat[SD V2.1]{\label{Fig:diffusion_tpr}\includegraphics[width=.15\linewidth]{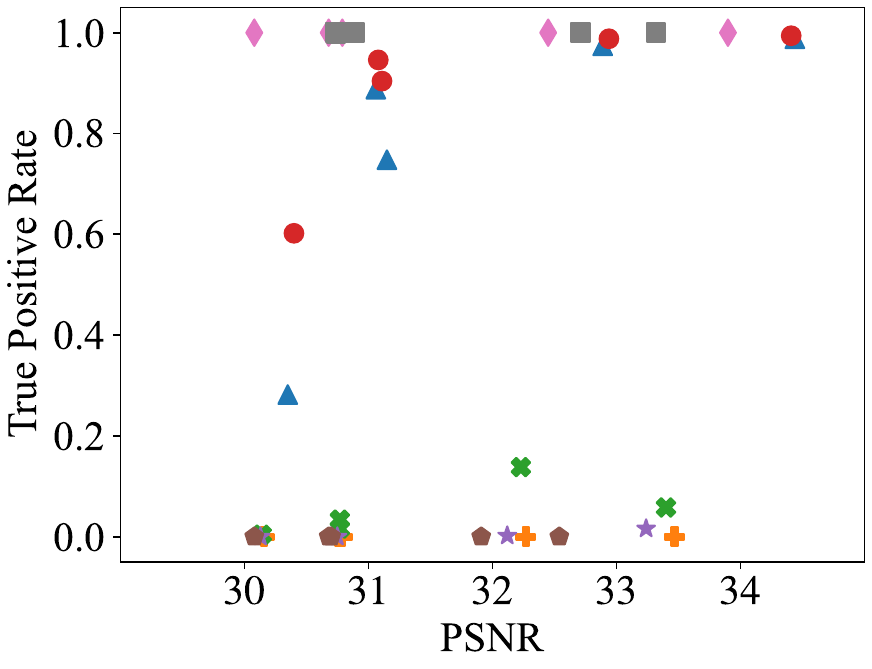}}\hspace{0.005\linewidth}

 {\label{Fig:acc_legend}\includegraphics[width=.16\linewidth]{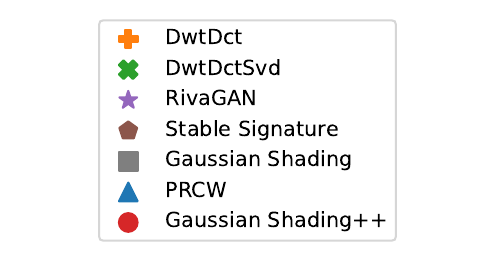}}
 \subfloat[VAE B]{\label{Fig:bmshj_acc}\includegraphics[width=.15\linewidth]{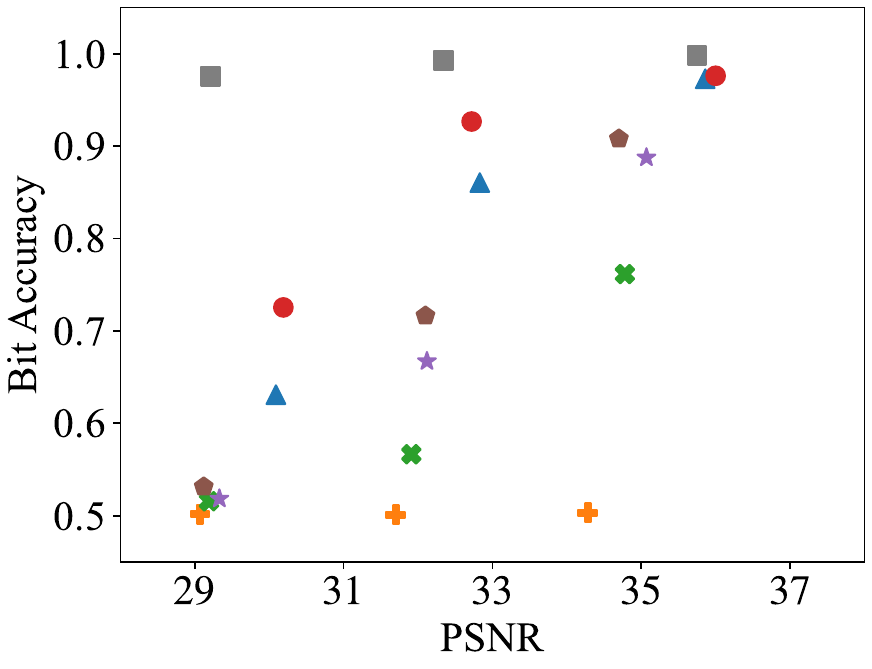}}\hspace{0.005\linewidth}
 \subfloat[VAE C]{\label{Fig:cheng_acc}\includegraphics[width=.15\linewidth]{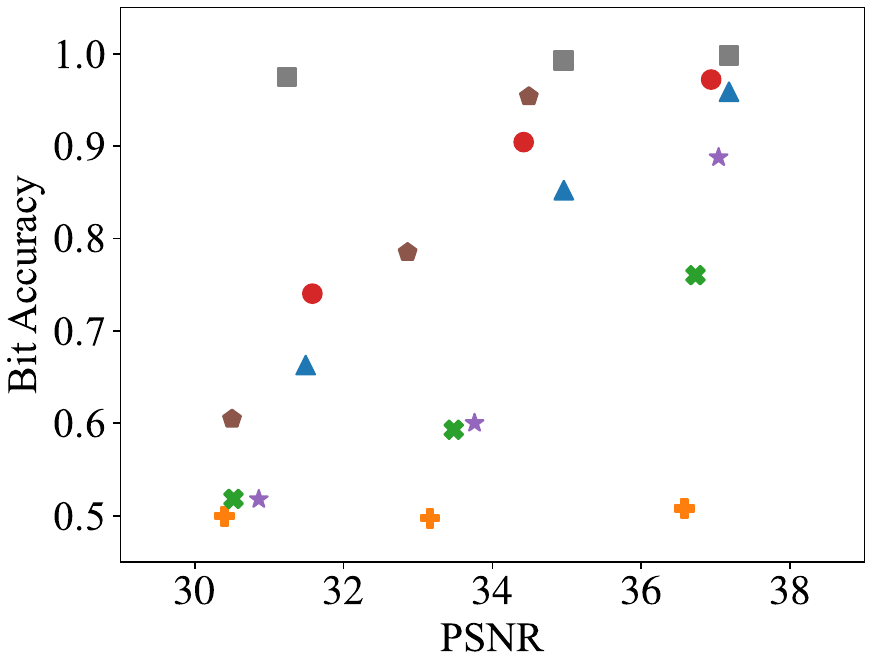}}\hspace{0.005\linewidth}
 \subfloat[VQ VAE]{\label{Fig:vq_acc}\includegraphics[width=.15\linewidth]{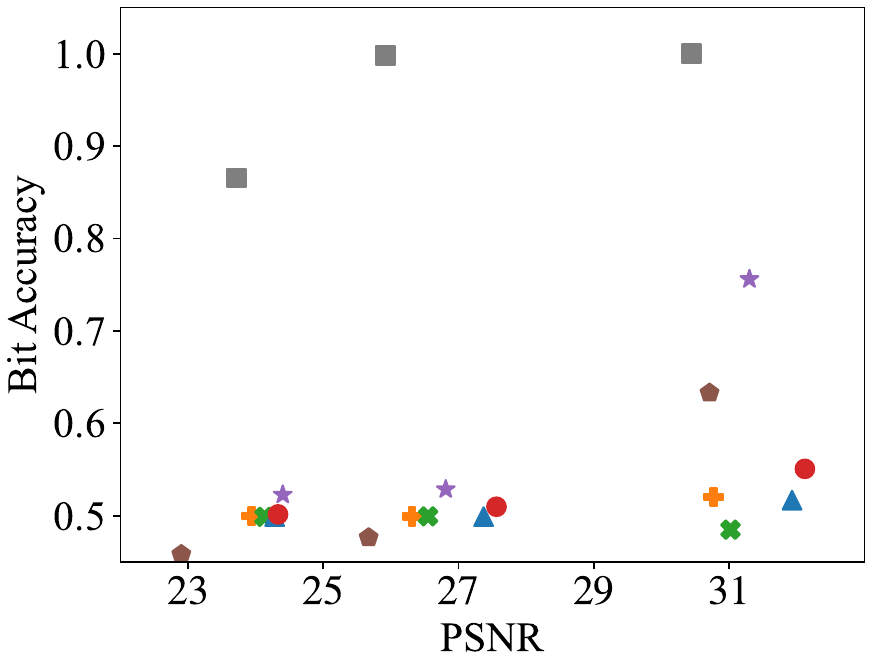}}\hspace{0.005\linewidth}
 \subfloat[KL VAE]{\label{Fig:kl_acc}\includegraphics[width=.15\linewidth]{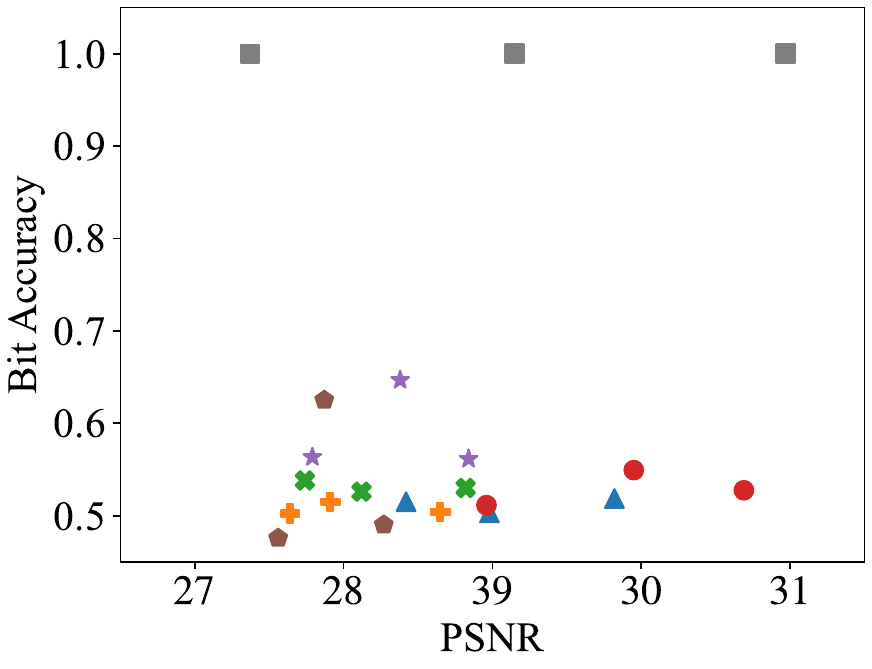}}\hspace{0.005\linewidth}
 \subfloat[SD V2.1]{\label{Fig:diffusion_acc}\includegraphics[width=.15\linewidth]{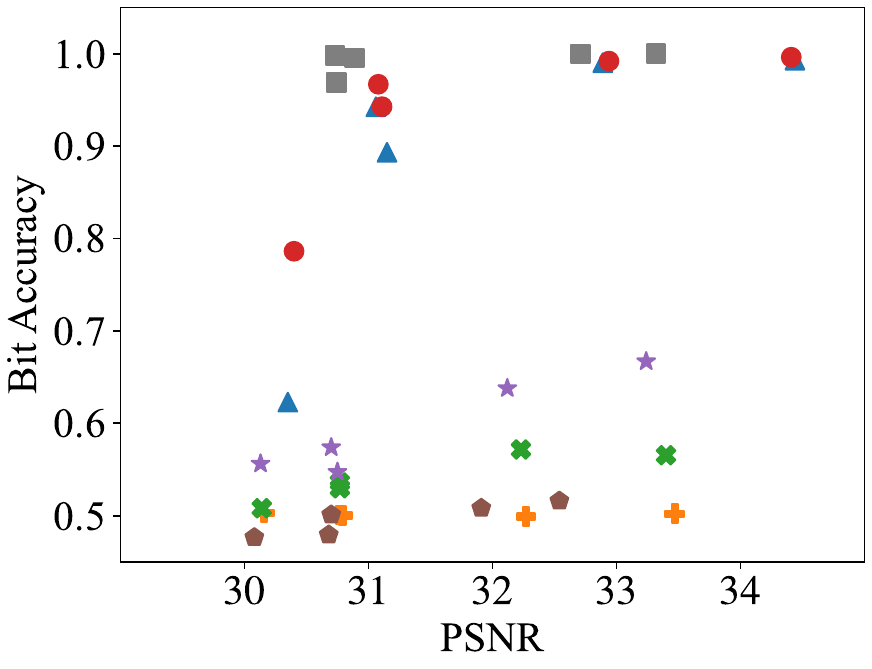}}\hspace{0.005\linewidth}
    \caption{Comparison of robustness under removal attacks. The first row presents the TPR of watermark methods, while the second row shows their bit accuracy.
    For VAE B~\cite{balle2018variational} and VAE C~\cite{cheng2020learned}, we select three strength levels $ quality= 2, 4,6 $.
For VQ VAE~\cite{esser2021taming} and KL VAE~\cite{rombach2022high}, we choose three strength levels $f=4, 8, 16$.
For SD V2.1~\cite{rombach2022high}, we set five removal steps $t_{step}=10, 25, 50, 100, 200$.
    }
    \label{fig:com_attack}
\end{figure*}

\subsection{Comparison to Baselines}
In this section, we compare the performance of Gaussian
Shading++ with baselines on SD V2.1. Robustness is evaluated from two aspects: traditional noise distortions and removal attacks. Additionally, performance-lossless characteristic is demonstrated in terms of visual quality and the distribution of latent representations. Finally, we compare PRCW~\cite{prc} by evaluating the performance under varying \textit{guidance\_scale}.

\subsubsection{Robustness to Noise}
We conduct tests on 500 generated images for each
method respectively. As shown in Tab.~\ref{tab:com_noise}, Gaussian Shading++ achieves near-optimal performance under various noise conditions, second only to Gaussian Shading~\cite{gs}. It outperforms the state-of-the-art method PRCW~\cite{prc}, and even improves bit accuracy by up to 4\% under Gaussian Blur. This is because, on the one hand, the insufficient robustness of the PRC Channel leads to decoding failure of the $seed$, and the fewer watermark copies in the GS Channel compared to Gaussian Shading~\cite{gs} result in reduced robustness. On the other hand, the GS Channel in Gaussian Shading++ exhibits stronger robustness than PRCW~\cite{prc} when the PRC Channel successfully decodes. As a result, the overall performance of Gaussian Shading++ is between that of Gaussian Shading~\cite{gs} and PRCW~\cite{prc}. Considering the challenges of key management in Gaussian Shading~\cite{gs}, Gaussian Shading++ is the optimal solution for practical deployment.

\subsubsection{Removal Attack}
Besides traditional noise distortions, recent studies~\cite{zhao2024invisible,an2024waves} emphasize the necessity for watermarking methods to resist removal attacks.  The main idea is to reconstruct watermarked images through networks to erase watermark signals. We consider two types of removal networks and evaluate five methods: VAE (VAE B~\cite{balle2018variational}, VAE C~\cite{cheng2020learned}, VQ VAE~\cite{esser2021taming}, KL VAE~\cite{rombach2022high}), and SD V2.1~\cite{rombach2022high}. For each watermarking method, we evaluate 500 images under different levels of each removal attack. The experimental results are shown in Fig.~\ref{fig:com_attack}. Post-processing-based~\cite{cox2007digital, zhang2019robust} and fine-tuning-based~\cite{fernandez2023stable} methods struggle to resist removal attacks since their watermark signals are encoded in the residuals. For latent-representation-based methods, Gaussian Shading~\cite{gs} demonstrates strong robustness against all removal attacks, whereas TRW~\cite{wen2023tree}, PRCW~\cite{prc}, and Gaussian Shading++ struggle to resist VAE-based removal attacks. This is because, although the GS Channel exhibits strong robustness, the PRC Channel consistently decodes an incorrect $seed$, leading to extraction errors in the GS Channel.

\begin{table}[htbp]
  \centering
    \caption{The $t$-test results for visual quality on SD V2.1. \textbf{Bold} represents the best, \underline{underline} represents the second best.
    }
   \resizebox{0.48\textwidth}{!}{
  \begin{tabular}{@{}ccc@{}}
    \toprule
    \multirow{2}{*}{\textbf{Methods}} &
		\multicolumn{2}{c}{\textbf{Metrics}} \\
		\cmidrule{2-3}
		&\textbf{FID ($t$-value $\downarrow$)}  & \textbf{CLIP-Score ($t$-value $\downarrow$)}\\
    \midrule
    Stable Diffusion&25.23$\pm$.18&0.3629$\pm$.0006\\
    \midrule
    DwtDct~\cite{cox2007digital} &24.97$\pm$.19 (3.026)&0.3617$\pm$.0007 (3.045)\\
    DwtDctSvd~\cite{cox2007digital} &24.45$\pm$.22 (8.253)&0.3609$\pm$.0009 (4.452)\\
    RivaGAN~\cite{zhang2019robust} &24.24$\pm$.16 (12.29)&0.3611$\pm$.0009 (4.259)\\
    Stable Signature~\cite{fernandez2023stable} &25.45$\pm$.18 (2.477)&0.3622$\pm$.0027 (\underline{0.7066})\\
    TRW~\cite{wen2023tree} &25.43$\pm$.13 (2.581)&0.3632$\pm$.0006 (0.8278)\\
    
    Gaussian Shading~\cite{gs} &25.15$\pm$.16 (1.005)&0.3624$\pm$.0006 (1.469)\\
    
    PRCW~\cite{prc} &25.22$\pm$.15 (\underline{0.1636})&0.3624$\pm$.0007 (0.8220)\\
  
    \textbf{Gaussian Shading++} &25.22$\pm$.10 (\textbf{0.1597})&0.3626$\pm$.0011 (\textbf{0.6228}) \\
    \bottomrule
    
  \end{tabular}}

  \label{tab:ttest}
\end{table}

\subsubsection{Performance Bias in Visual Quality} \label{sec:visual_quality}

To assess the performance bias from watermark embedding, we conduct a $t$-test. We generate $50,000/10,000$ images with SD V2.1 for each watermarking method, split into 10 groups of $5,000/1,000$ images. The FID~\cite{heusel2017gans}/CLIP-Score~\cite{radford2021learning} is computed for each group, and the mean $\mu_s$ is obtained. Similarly, we generate $50,000/10,000$ non-watermarked images, evaluate FID/CLIP-Score across 10 groups, and calculate $\mu_0$.
For FID, $5,000$ images are sampled from MS-COCO-2017~\cite{lin2014microsoft} validation set for score computation. For CLIP-Score, OpenCLIP-ViT-G~\cite{cherti2023reproducible} is used to assess image-text relevance.

If the model performance is maintained, then $\mu_s$ and $\mu_0$ should be statistically close to each other. Therefore, the hypotheses are $H_0 : \mu_s = \mu_0, H_1 : \mu_s \neq \mu_0$.
A lower $t$-value suggests a higher likelihood of $H_0$ holding. If it exceeds a threshold, $H_0$ is rejected, indicating model performance degradation. 
As shown in Tab.~\ref{tab:ttest}, most existing method~\cite{cox2007digital, zhang2019robust,fernandez2023stable,wen2023tree,gs} significantly impact model generation performance from a statistical perspective. In contrast, Gaussian Shading++ achieves the smallest $t$-value for both FID and CLIP-Score, indirectly confirming its performance-lossless characteristic.

\begin{table*}
  \centering
    \caption{Comparison of Latent Distributions Under Different Statistics and Statistical Tests.
We highlight clearly abnormal statistics and failed normality tests in \textbf{bold}, while \underline{underline} indicates tests with low confidence (i.e., small p-values close to zero).}
  \label{tab:latent distribution}
   \resizebox{\textwidth}{!}{
  \begin{tabular}{@{}cccccccc@{}}
    \toprule
    \multirow{2}{*}{\textbf{Methods}} &
		\multicolumn{7}{c}{\textbf{Statistics and Statistical Tests (Statistic / p-value $\uparrow$)}} \\
		\cmidrule{2-8}
		&\textbf{Mean}& \textbf{Frobenius norm} &\textbf{K-S}&\textbf{Shapiro–Wilk}&\textbf{Cramér–von Mises } &\textbf{Jarque–Bera}&\textbf{D’Agostino’s K-squared} \\
    \midrule

    Standard Sampling &-4.4599e-6 / 0.9990&57.9245 / - &1.5989e-5 / 0.8910&0.9999 / 1.0000 &0.0212 / 0.9957&0.0435 / 0.9785&0.0435 / 0.9785\\
    \midrule
    Gaussian Shading~\cite{gs} &\textbf{0.3989} / \textbf{0.0000}&\textbf{183.4432} / -&\textbf{0.5000} / \textbf{0.0000}&\textbf{0.7335} / \textbf{0.0000}&\textbf{5.4613e+7} / \textbf{0.0000}&\textbf{9.0466e+8} / \textbf{0.0000}&\textbf{3.7927e+8} / \textbf{0.0000}\\
    
    TRW~\cite{wen2023tree} &-8.6425e-6 / 0.9980&57.9628 / - &\textbf{0.0003} / \textbf{0.0000}&0.9999 / 1.0000 &\textbf{60.4934} / \textbf{0.0000}&0.0055 / 0.9972&0.0055 / 0.9972\\
    
    PRCW~\cite{prc} &1.966e-5 / 0.9956&57.9865 / -&2.3074e-5 / 0.4877&0.9999 / 1.0000&0.1636 / 0.3507&\underline{5.2902} / \underline{0.0710}&\underline{5.2904} / \underline{0.0710}\\
    
    \textbf{Gaussian Shading++} &-2.6283e-5 / 0.9941&57.9882 / -&2.2996 / 0.4922&0.9999 / 1.0000&0.1093 / 0.5413&0.8647 / 0.6489&0.8647 / 0.6489\\
    \bottomrule
    
  \end{tabular}}
\end{table*}

\subsubsection{Latent Distribution}\label{sec:Latent Distribution}

To evaluate whether the latent space distribution of our watermarking scheme conforms to the standard Gaussian distribution \( \mathcal{N}(0, I) \), we generated a set of latent vectors containing a fixed watermark message. We computed the sample mean and the Frobenius norm of the covariance matrix, and conducted five hypothesis tests to assess normality: the Kolmogorov–Smirnov (K-S) test ~\cite{smirnov1939estimation}, the Shapiro–Wilk test ~\cite{shapiro1965analysis}, the Cramér–von Mises test ~\cite{anderson1962distribution}, the Jarque–Bera test ~\cite{jarque1987test}, and D’Agostino’s K-squared test ~\cite{d1990suggestion}. Specifically, the K-S test measures the maximum deviation between the empirical and theoretical cumulative distribution functions; the Shapiro–Wilk test evaluates the correlation between the data and expected normal scores; the Cramér–von Mises test quantifies the integrated squared difference between distributions; and the Jarque–Bera and D’Agostino’s tests assess skewness and kurtosis.

Our baseline methods include: standard Gaussian sampling via `torch.randn', Gaussian shading~\cite{gs} with a fixed key, TRW~\cite{wen2023tree}, and PRCW~\cite{prc}. For all comparison methods, we fix the same watermark message and generate \( 80{,}000 \) latent vectors of shape \( [4, 64, 64] \) for evaluation.

The experimental results in Tab.~\ref{tab:latent distribution} show that Gaussian Shading ~\cite{gs}, under a fixed key and fixed watermark message, fails all normality tests. This is expected, as the output of its watermark ciphertext is deterministic under these fixed conditions, causing the latent space samples to concentrate in a limited subspace. 
TRW ~\cite{wen2023tree} fails two of the tests—the Kolmogorov–Smirnov and the Cramér–von Mises tests—which is also reasonable, as TRW does not explicitly enforce distributional preservation in the latent space. 

PRCW ~\cite{prc} passes all the tests; however, its \(p\)-values in the Jarque–Bera and D’Agostino’s K-squared tests are relatively low, representing a weak pass. We hypothesize that this is because the watermark ciphertext in PRCW takes the form \( G \cdot (testbits \,\|\, m \,\|\, r) \oplus e \). When both $testbits$ and \( m \) are fixed, the only source of randomness comes from \( r \), which may be insufficient under large sample sizes and could lead to slight deviations from normality. 

In contrast, our proposed method passes all normality tests perfectly. Since the watermark header is computed as \( G \cdot seed \oplus e \), and \( seed \) is truly sampled at random, the resulting pseudorandomness is theoretically guaranteed by the pseudorandomness of the underlying PRC construction.

\begin{table*}
  \centering
    \caption{Robustness comparison under different guidance\_scale during generation. We control the FPR at $10^{-6}$, and evaluate the TPR / bit accuracy for SD V2.1. The guidance\_scale is fixed to $3$ during inversion. With no background representing PRCW~\cite{prc} and a gray background representing Gaussian Shading++. }
  \label{tab:com_gs}
  \begin{tabular}{@{}cccccccc@{}}
    \toprule
    \multirow{2}{*}{\textbf{Guidance Scale}} &
		\multicolumn{7}{c}{\textbf{Noise}} \\
		\cmidrule{2-8}
		&\textbf{None}& \textbf{JPEG} &\textbf{Brightness}&\textbf{GauBlur}&\textbf{GauNoise} &\textbf{MedFilter}&\textbf{Resize} \\
    \midrule
    \multirow{2}{*}{3}
    &\textbf{1.000} / \textbf{1.000}
&0.992 / 0.996
&0.982 / 0.993
&0.976 / \textbf{0.999}
&0.996 / 0.997
&0.98 / 0.992
&0.998 / 0.999\\
 &\cellcolor{gray!30}\textbf{1.000} / \textbf{1.000}
&\cellcolor{gray!30}\textbf{0.994} / \textbf{0.997}
&\cellcolor{gray!30}\textbf{0.990} /\textbf{ 0.995}
&\cellcolor{gray!30}\textbf{0.998} / 0.993
&\cellcolor{gray!30}\textbf{0.998} / \textbf{0.999}
&\cellcolor{gray!30}\textbf{0.994} / \textbf{0.994}
&\cellcolor{gray!30}\textbf{1.000} / \textbf{1.000}\\

       \multirow{2}{*}{6}
    &\textbf{1.000} / \textbf{1.000}
&0.964 / 0.984
&0.964 / 0.984
&0.922 / 0.964
&0.946 / 0.972
&0.910 / 0.964
&0.994 / 0.996\\
   &\cellcolor{gray!30}\textbf{1.000} / \textbf{1.000}
&\cellcolor{gray!30}\textbf{0.982} / \textbf{0.990}
&\cellcolor{gray!30}\textbf{0.982} / \textbf{0.990}
&\cellcolor{gray!30}\textbf{0.984} / \textbf{0.980}
&\cellcolor{gray!30}\textbf{0.970} / \textbf{0.984}
&\cellcolor{gray!30}\textbf{0.974} / \textbf{0.980}
&\cellcolor{gray!30}\textbf{0.996} / \textbf{0.998}\\

    \multirow{2}{*}{9}
    &\textbf{1.000} / \textbf{1.000}
&0.898 / 0.953
&0.952 / 0.977
&0.784 / 0.907
&0.790 / 0.892
&0.832 / 0.919
&0.980 / 0.992\\
   &\cellcolor{gray!30}\textbf{1.000} / \textbf{1.000}
&\cellcolor{gray!30}\textbf{0.966} / \textbf{0.979}
&\cellcolor{gray!30}\textbf{0.964} / \textbf{0.981}
&\cellcolor{gray!30}\textbf{0.938} / \textbf{0.952}
&\cellcolor{gray!30}\textbf{0.848} / \textbf{0.919}
&\cellcolor{gray!30}\textbf{0.938} / \textbf{0.960}
&\cellcolor{gray!30}\textbf{0.996} / \textbf{0.997}\\

    \multirow{2}{*}{12}
    &0.996 / 0.998
&0.824 / 0.917
&0.924 / 0.963
&0.622 / 0.813
&0.594 / 0.798
&0.736 / 0.875
&0.948 / 0.976\\
   &\cellcolor{gray!30}\textbf{1.000} / \textbf{1.000}
&\cellcolor{gray!30}\textbf{0.914} / \textbf{0.952}
&\cellcolor{gray!30}\textbf{0.960} / \textbf{0.978}
&\cellcolor{gray!30}\textbf{0.822} / \textbf{0.919}
&\cellcolor{gray!30}\textbf{0.724} / \textbf{0.885}
&\cellcolor{gray!30}\textbf{0.890} / \textbf{0.932}
&\cellcolor{gray!30}\textbf{0.972} / \textbf{0.984}\\

    \multirow{2}{*}{15}
    &0.994 / \textbf{0.998}
&0.746 / 0.877
&0.874 / 0.945
&0.434 / 0.717
&0.434 / 0.709
&0.582 / 0.806
&0.900 / 0.960\\
   &\cellcolor{gray!30}\textbf{0.996} / \textbf{0.998}
&\cellcolor{gray!30}\textbf{0.838} / \textbf{0.912}
&\cellcolor{gray!30}\textbf{0.922} / \textbf{0.958}
&\cellcolor{gray!30}\textbf{0.782} / \textbf{0.867}
&\cellcolor{gray!30}\textbf{0.536} / \textbf{0.761}
&\cellcolor{gray!30}\textbf{0.802} / \textbf{0.886}
&\cellcolor{gray!30}\textbf{0.974} / \textbf{0.982}\\
\bottomrule
  \end{tabular}
\end{table*}

\subsubsection{Performance across Guidance\_scale}
Our experiments have confirmed that Gaussian Shading++ achieves performance-lossless watermarking with a fixed key and enables third-party verifiability. To further validate its practicality, evaluating its performance under varying \textit{guidance\_scale} is necessary. We primarily compare Gaussian Shading++ with PRCW~\cite{prc}. For each value of \textit{guidance\_scale} set to $3, 6, 9, 12, 15$, we generate $500$ watermarked images for evaluation. Since the generation parameters are unknown during the inversion phase, the \textit{guidance\_scale} is fixed at $3$ for all inverse processes. The results in Tab.~\ref{tab:com_gs} demonstrate that Gaussian Shading++ consistently outperforms PRCW~\cite{prc} across all settings. This performance gain is attributed to our modeling of the AWGN Channel across the entire generation and inversion process, and the use of soft decision decoding. These results highlight the suitability of Gaussian Shading++ for real-world deployments where generation parameters may vary.

\begin{table*}[t]
    \centering
        \caption{Ablation study on key modules of Gaussian Shading++. We control the FPR at $10^{-6}$, and evaluate the TPR/bit accuracy under various noise distortions for SD V2.1.\textbf{Bold} represents the best, \underline{underline} represents the second best. }
    \label{tab:ab}

    \begin{tabular}{c c c | c c c c c c c}
        \toprule
        \multicolumn{3}{c|}{\textbf{Module}} & \multicolumn{7}{c}{\textbf{Noise}} \\
        \cmidrule(lr){1-3} \cmidrule(lr){4-10}
        \textbf{Encode} & \textbf{Decode} & \textbf{PRC Loc} & \textbf{None} & \textbf{JPEG} & \textbf{Brightness} & \textbf{GauBlur} & \textbf{GauNoise} & \textbf{MedFilter} & \textbf{Resize} \\
        \midrule
        
        \textbf{BCH} & \textbf{SDD}  & \textbf{front} 
        &\underline{0.994} / \underline{0.997}
&0.588 / 0.794
&0.820 / 0.910
&0.008 / 0.503
&0.644 / 0.820
&0.062 / 0.531
&0.732 / 0.866\\
        
        \textbf{PRC}  & \textbf{HDD} & \textbf{front} 
        &\textbf{1.000} / \textbf{1.000}
&0.966 / \underline{0.977}
&0.980 / \underline{0.987}
&0.970 / 0.948
&\underline{0.914} / 0.952
&\textbf{0.964} / 0.959
&\textbf{0.996} / \underline{0.995} \\

       \textbf{PRC}  & \textbf{SDD-LLR} & \textbf{front} 
        &\textbf{1.000} / \textbf{1.000}
&\underline{0.970} / \textbf{0.984}
&\textbf{0.990} / \textbf{0.994}
&\textbf{0.974} / \textbf{0.972}
&0.910 / \underline{0.953}
&\textbf{0.964} / \textbf{0.975}
&0.990 / 0.994\\
        
        \textbf{PRC} & \textbf{SDD} & \textbf{top} 
        & \textbf{1.000} / \textbf{1.000}
&0.940 / 0.967
&0.960 / 0.979
&0.906 / 0.946
&0.872 / 0.932
&\underline{0.936} / 0.962
&0.992 / \underline{0.995} \\
        
        \textbf{PRC} & \textbf{SDD} & \textbf{bottom} 
        & \textbf{1.000} / \textbf{1.000}
&0.958 / 0.974
&0.966 / 0.981
&0.934 / \underline{0.959}
&\underline{0.914} / 0.952
&0.932 / 0.959
&\underline{0.994} / \textbf{0.996} \\
        
        \textbf{PRC} & \textbf{SDD} & \textbf{back} 
        & \textbf{1.000} / \textbf{1.000}
&0.948 / 0.973
&0.966 / 0.982
&0.846 / 0.920
&0.888 / 0.942
&0.906 / 0.950
&0.988 / 0.993 \\
\midrule
        \textbf{PRC} & \textbf{SDD} & \textbf{front} 
        & \textbf{1.000} / \textbf{1.000}
&\textbf{0.974} / \textbf{0.984}
&\underline{0.974} / 0.986
&\textbf{0.974} / \textbf{0.972}
&\textbf{0.918} / \textbf{0.956}
&\textbf{0.964} / \underline{0.974}
&\textbf{0.996} / \textbf{0.996} \\
        \bottomrule
    \end{tabular}

\end{table*}

\begin{figure*}[t]
    \centering
        
    \subfloat[JPEG]{\label{Fig:jp}\includegraphics[width=.155\linewidth]{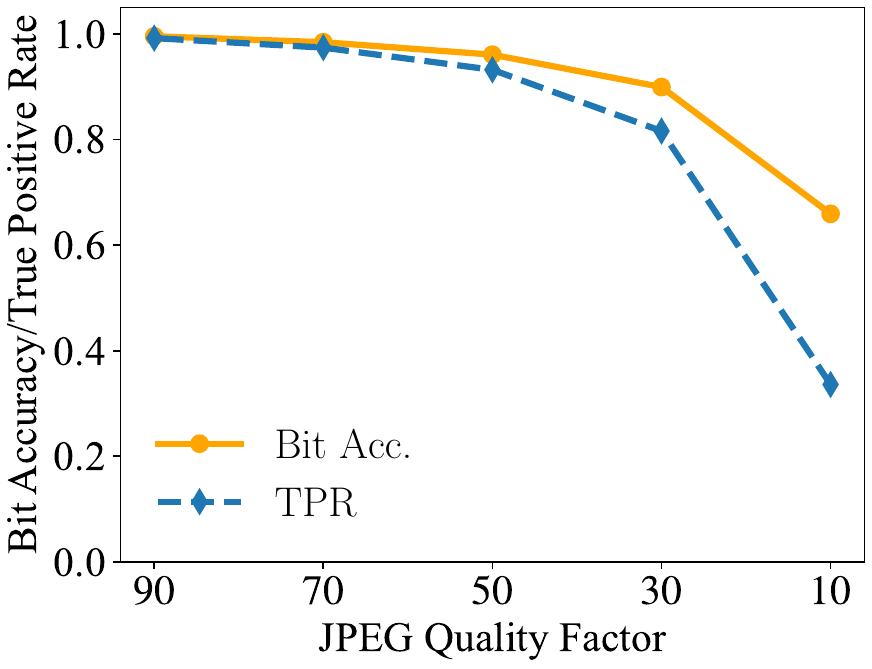}}\hspace{0.005\linewidth}
     \subfloat[Brightness]{\label{Fig:br}\includegraphics[width=.155\linewidth]{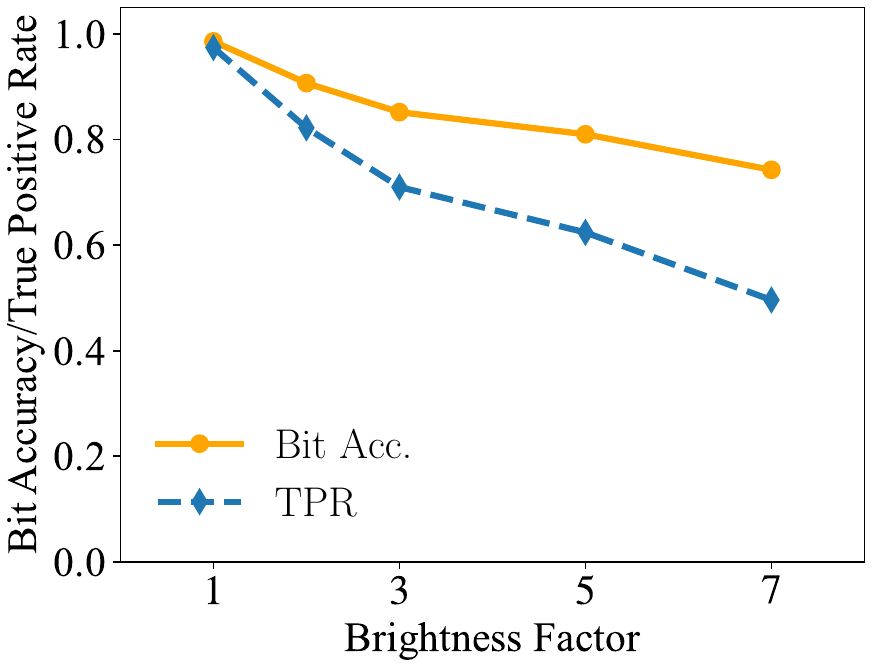}}\hspace{0.005\linewidth}
 \subfloat[GauBlur]{\label{Fig:gb}\includegraphics[width=.155\linewidth]{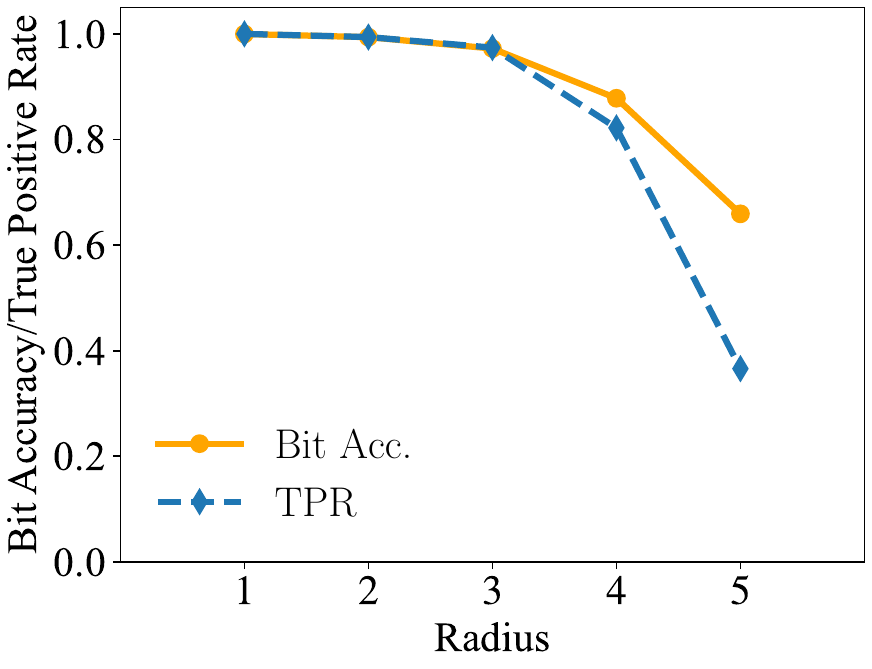}}\hspace{0.005\linewidth}
  \subfloat[GauNoise]{\label{Fig:gn}\includegraphics[width=.155\linewidth]{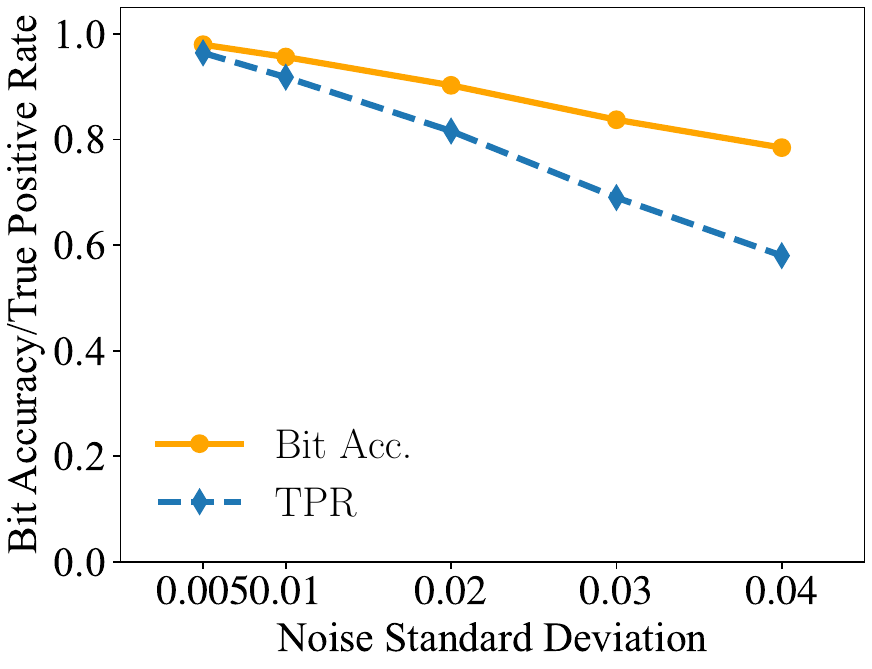}}\hspace{0.005\linewidth}
 \subfloat[MedFilter]{\label{Fig:mf}\includegraphics[width=.155\linewidth]{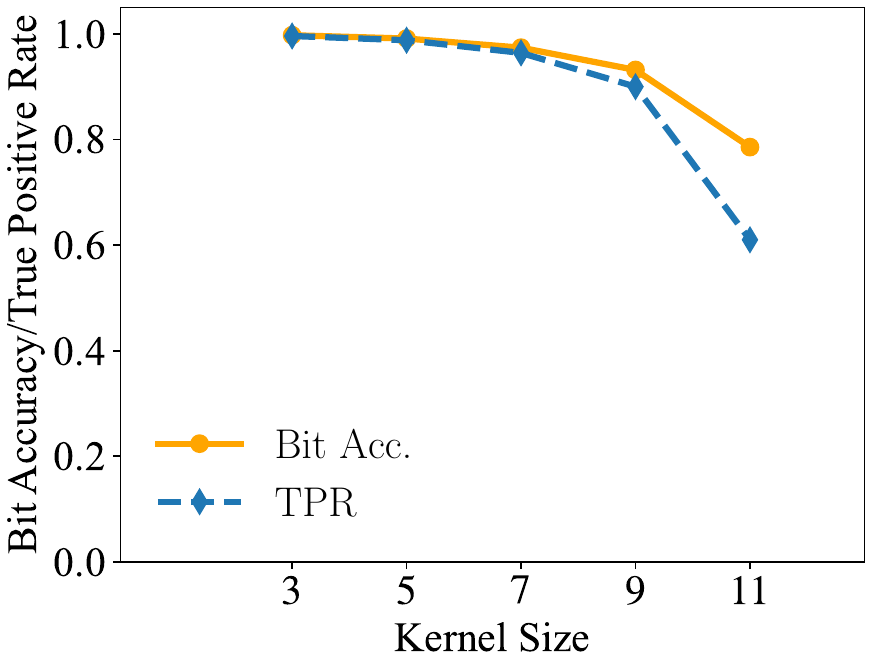}}\hspace{0.005\linewidth}
 \subfloat[Resize]{\label{Fig:rs}\includegraphics[width=.155\linewidth]{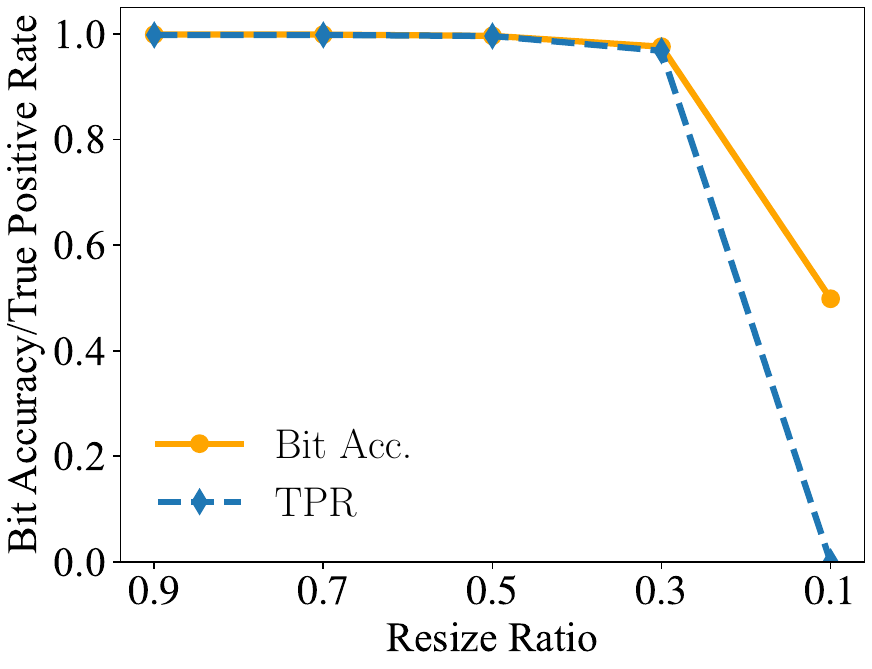}}\hspace{0.005\linewidth}

 \caption{Ablation study on different noise strengths.}
    \label{fig:ablation}

\end{figure*}

\subsection{Ablation Studies}

In this section, we conduct ablation studies on key modules of Gaussian Shading++, including the pseudorandom error-correcting codes (PRC), soft decision decoding (SDD), and the location of the PRC Channel in the latent space (PRC Loc). The last row of Tab.~\ref{tab:ab} represents the default configuration of Gaussian Shading++. Finally, we evaluate the performance of Gaussian Shading++ under various noise strengths.

\subsubsection{Effect of PRC}
To evaluate the effectiveness of encoding the $seed$ in the PRC Channel using PRC codes, we replace PRC with Bose–Chaudhuri–Hocquenghem codes (BCH)~\cite{bose1960class}. As shown in the first row of Tab.~\ref{tab:ab}, this substitution leads to a notable decline in robustness. Furthermore, BCH codes lack pseudorandomness, making it challenging to achieve provable performance-lossless.

\subsubsection{Effect of SDD}
To assess the effectiveness of the soft decision decoding, we remove the posterior estimation step in the GS Channel and instead apply hard decision decoding (HDD), as used in Gaussian Shading~\cite{gs}. Specifically, watermark bits are directly inferred from the sign of elements in $z'^s_{T}$, and final watermark recovery is performed via majority voting over bits at the same positions. As shown in the second row of Tab.~\ref{tab:ab}, using HDD leads to a noticeable decline in overall robustness, particularly under filtering-based distortions such as Gaussian Blur and Median Filtering. These results highlight the advantage of our soft decision decoding strategy.

Additionally, we evaluate the soft decision decoding method using full LLR (SDD-LLR), as defined in Eq.~\ref{eq:full_llr}. As shown in the third row of Tab.~\ref{tab:ab}, SDD-LLR, which utilizes the arctanh function, performs similarly to the soft decision decoding, which applies a first-order approximation of arctanh. Nevertheless, soft decision decoding offers reduced computational complexity.

\subsubsection{PRC Channel Location}
Since the entire latent space is divided into the PRC Channel and the GS Channel, how these two are combined is crucial. We mainly discuss the position of the PRC Channel, with the remaining part being the GS Channel. Considering the latent space dimensions of \( [4, 64, 64] \), we define four positions for the PRC Channel: top, occupying the upper half, i.e., \([0:4, 0:32, 0:64]\); bottom, occupying the lower half, i.e., \([0:4, 32:64, 0:64]\); front, occupying the first two channels, i.e., \([0:2, 0:64, 0:64]\); back, occupying the last two channels, i.e., \([2:4, 0:64, 0:64]\). The experimental results, shown in the last four rows of Tab.~\ref{tab:ab}, indicate that when the PRC Channel occupies the first two channels (front), Gaussian Shading++ achieves the best performance, which is also our default setting.

\subsubsection{Noise Strengths}
To further test the robustness, we conduct experiments using different intensities of noises. As show in  Fig.~\ref{fig:ablation}, performance declines with higher intensities. However, for Brightness, Gaussian Noise, and Median Filter, even at high intensities, the bit accuracy remains above $75\%$.

\section{ Conclusion, Limitations, and Future Work}
We propose Gaussian Shading++, a performance-lossless watermarking method for diffusion models designed for practical deployment. First, we propose a double-channel design, utilizing the PRC Channel to encode the random seed required for the pseudorandomization of the GS Channel. This enables performance-lossless watermarking when the watermark key is fixed, addressing key management challenges in real-world applications. Second, we model the entire generation and inversion process as an additive white Gaussian noise channel and propose a novel soft decision decoding strategy for the maximum likelihood decoding in REP codes, ensuring strong robustness even when generation parameters vary. Third, we introduce the public-key signature ECDSA, extending watermark verification to any third party while providing resistance against forgery attacks. Extensive experiments validate that our method outperforms state-of-the-art approaches in both robustness and performance losslessness, making it a more practical watermarking solution for real-world deployment.

Although Gaussian Shading++ demonstrates excellent performance in practical deployment scenarios, there are still some limitations. First, while the GS Channel is sufficiently robust, the performance bottleneck lies in the PRC Channel. In scenarios where PRC~\cite{christ2024pseudorandom} performs poorly, such as erasure attacks from VQ VAE~\cite{esser2021taming}, the robustness of Gaussian Shading++ is also affected. Second, the anti-forgery capability of Gaussian Shading++ in publicly verifiable scenarios still has room for improvement, as it remains vulnerable to forgery when attackers use proxy models with similar parameters. Lastly, although Exact Inversion~\cite{Hong_2024_CVPR} improves extraction accuracy, its optimization process inevitably increases the computational cost of watermark extraction.

Therefore, future work will focus on three key aspects. First, to fully leverage the performance of the GS Channel, more advanced pseudorandom coding schemes~\cite{alrabiah2024ideal} should be introduced to construct a more robust watermark header. Second, exploring denoising sampling or latent representation decoding could help develop more secure and controllable watermarking methods to enhance anti-forgery capabilities in third-party verifiable scenarios. Third, faster inversion techniques~\cite{hong2024gradient} should be incorporated to reduce the computational cost of watermark extraction without sacrificing accuracy.

\bibliography{TPAMI}

\end{document}